%% file: 2023_icml_main.tex
\documentclass[nohyperref]{article}

\usepackage{microtype}
\usepackage{graphicx}

\usepackage{hyperref}

\usepackage[ruled,algo2e]{algorithm2e}

\usepackage[round]{natbib}

\input{packages/math_commands.tex}

\usepackage[accepted]{icml2023_new}

\usepackage{amsmath}
\usepackage{amssymb}
\usepackage{mathtools}
\usepackage{amsthm}
\usepackage{amstext}
\usepackage{mathrsfs}

\usepackage{graphicx}
\usepackage{xcolor}
\usepackage{arydshln}
\usepackage{float}
\usepackage{lmodern}
\usepackage{booktabs}
\usepackage{comment}

\usepackage{titletoc}
  
\usepackage{subcaption}

\usepackage[capitalize,noabbrev]{cleveref}

\theoremstyle{plain}
\newtheorem{theorem}{Theorem}[section]
\newtheorem{proposition}[theorem]{Proposition}
\newtheorem{lemma}[theorem]{Lemma}
\newtheorem{corollary}[theorem]{Corollary}
\theoremstyle{definition}
\newtheorem{definition}[theorem]{Definition}

\theoremstyle{remark}

\newenvironment{packeditemize}{
\begin{list}{$\bullet$}{
\setlength{\labelwidth}{8pt}
\setlength{\itemsep}{0pt}
\setlength{\leftmargin}{\labelwidth}
\addtolength{\leftmargin}{\labelsep}
\setlength{\parindent}{0pt}
\setlength{\listparindent}{\parindent}
\setlength{\parsep}{0pt}
\setlength{\topsep}{3pt}}}{\end{list}}
\usepackage[textsize=tiny]{todonotes}
\usepackage{wrapfig}
\usepackage[title]{appendix}

\newcommand{\alg}{\texttt{2D-Shapley}} 
\newcommand{\algmc}{\texttt{2D-Shapley-MC}} 

\newcommand{\algknn}{\texttt{2D-Shapley-KNN}} 

\newcommand{\oned}{\texttt{1D-Shapley}} 

\usepackage{hyperref}
\usepackage{url}
\newtheorem{axiom}{Axiom}
\newtheorem{example}{Example}

\newcommand{\xiangyu}[1]{\textbf{\textcolor{red}{[Xiangyu: #1]}}}

\newcommand{\hoang}[1]{\textbf{\textcolor{violet}{[Hoang: #1]}}}
\newcommand{\edit}[1]{{\textcolor{black}{#1}}}

\usepackage{subfiles} 

\icmltitlerunning{\alg: A Framework for Fragmented Data Valuation}

\begin{document}

\twocolumn[
\icmltitle{\alg: A Framework for Fragmented Data Valuation}



\icmlsetsymbol{equal}{*}
\begin{icmlauthorlist}
\icmlauthor{Zhihong Liu}{equal,yyy}
\icmlauthor{Hoang Anh Just}{equal,comp}
\icmlauthor{Xiangyu Chang}{yyy}
\icmlauthor{Xi Chen}{sch}
\icmlauthor{Ruoxi Jia}{comp}
\end{icmlauthorlist}

\icmlaffiliation{yyy}{Center for Intelligent Decision-Making and Machine Learning, Department of Information Systems and Intelligent Business, School of Management, Xi'an Jiaotong University, Xi'an, 710049, China.}
\icmlaffiliation{comp}{Bradley Department of Electrical and Computer Engineering, Virginia Tech, Virginia, USA.}
\icmlaffiliation{sch}{Department of Technology, Operations, and Statistics, Stern School of Business, New York University, New York, 10012, USA}
\icmlcorrespondingauthor{Xiangyu Chang}{ xiangyuchang@xjtu.edu.cn}
\icmlcorrespondingauthor{Xi Chen}{xc13@stern.nyu.edu}
\icmlcorrespondingauthor{Ruoxi Jia}{ruoxijia@vt.edu}

\icmlkeywords{Machine Learning, ICML, XAI}

\vskip 0.3in]


\printAffiliationsAndNotice{\icmlEqualContribution}  

\begin{abstract}
Data valuation---quantifying the contribution of individual data sources to certain predictive behaviors of a model---is of great importance to enhancing the transparency of machine learning and designing incentive systems for data sharing. 
Existing work has focused on evaluating data sources with the shared feature or sample space. 
How to valuate fragmented data sources of which each only contains partial features and samples remains an open question. 
We start by presenting a method to calculate the counterfactual of removing a fragment from the aggregated data matrix. 
Based on the counterfactual calculation, we further propose $\alg$, a theoretical framework for fragmented data valuation that uniquely satisfies some appealing axioms in the fragmented data context. 
$\alg$ empowers a range of new use cases, such as selecting useful data fragments, providing interpretation for sample-wise data values, and fine-grained data issue diagnosis.


\end{abstract}

\input{tex/intro}

\input{tex/related_work}

\input{tex/2d_shapley}

\input{tex/experiments}
\input{tex/conclusion}

\input{tex/acknowledge}


\bibliography{2023_icml_main}
\bibliographystyle{icml2023}

\onecolumn
\icmltitle{ \alg: A Framework for Fragmented Data Valuation \\ Supplementary Materials 
}



\input{tex/appendix}

%

\end{document}

%% file: packages/math_commands.tex
\usepackage{amsmath,amsfonts,bm}

\def\eqref#1{equation~\ref{#1}}

\def\1{\bm{1}}

\def\rvb{{\mathbf{b}}}

\def\rvx{{\mathbf{x}}}

\def\mA{{\bm{A}}}
\def\mB{{\bm{B}}}

\def\mO{{\bm{O}}}

\DeclareMathAlphabet{\mathsfit}{\encodingdefault}{\sfdefault}{m}{sl}
\SetMathAlphabet{\mathsfit}{bold}{\encodingdefault}{\sfdefault}{bx}{n}

\newcommand{\R}{\mathbb{R}}



%% file: tex/intro.tex
\section{Introduction}\label{sec:introduction}
Data are essential ingredients for building machine learning (ML) applications. 
The ability to quantify and measure the value of data is crucial to the entire lifecycle of ML: from cleaning poor-quality samples and tracking important ones to be collected during data preparation to setting proper proprieties over samples during training to interpret why certain behaviors of a model emerge during deployment. 
Determining the value of data is also central to designing incentive systems for data sharing and implementing current policies about the monetarization of personal data. 

\begin{figure}[t!]
    \centering
    \vspace{-0.5em}
    \includegraphics[width=0.42\textwidth]{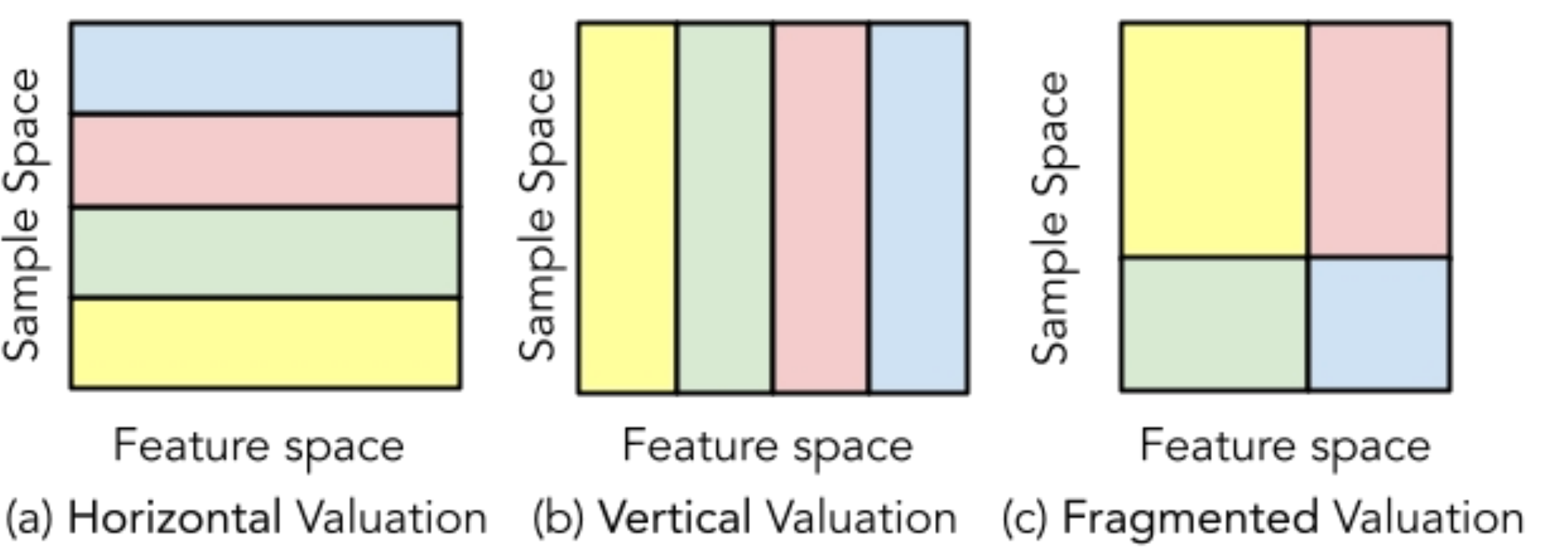}
     \vspace{-1em}
    \caption{Illustration of different data valuation settings based on how training set is partitioned among different data contributors. }
    \vspace{-1.5em}
    \label{fig:incomplete}
\end{figure}

Current literature of data valuation~\cite{jia2019towards,ghorbani2019data} has exclusively focused on valuing \textit{horizontally} partitioned data---in other words, each data source to be valued shares the same feature space. 
How to value \textit{vertically} partitioned data, where each data source provides a different feature but shares the same sample space, has been studied in the context of ML interpretability~\cite{covert2020understanding}.
However, none of these abstractions could fully capture the complexity of real-world scenarios, where data sources can have non-overlapping features and samples (termed as \textit{fragmented} data sources hereinafter). 
 
\begin{example}\label{Ex1}
Consider two banks, $B_1$ and $B_2$, and two e-commerce companies, $E_1$ and $E_2$, located in Region $1$ and $2$. 
These four institutions are interested in collaboratively building an ML model to predict users' credit scores with their data. 
Due to the geographical difference, $B_1$ and $E_1$ have a different user group from $B_2$ and $E_2$. 
Also, due to the difference in business, $B_1$ and $B_2$ provide different features than what $E_1$ and $E_2$ can offer. 
Overall, the four institutions partition the aggregated data horizontally and vertically, as illustrated by Figure~\ref{fig:incomplete}(c). 
\emph{How to quantify each institution's contribution to the joint model training?}

\end{example}

\begin{example}\label{Ex2}
Due to inevitable errors occurring during the data generation and collection processes, real-world data are seldom high quality. 
Suppose that a data analyst is interested in identifying some potentially erroneous entries in a dataset. 
Existing horizontal and vertical data valuation tools can help locate the rows or columns that could contain errors by returning the ones with the lowest values. 
Nevertheless, can we perform more fine-grained detection---\emph{e.g., how to pinpoint the coordinate of erroneous entries?} 
\end{example}

\begin{example}\label{Ex3} Horizontal data valuation is now widely used to explain the importance of each sample to a learning outcome~\cite{tang2021data,karlavs2022data}.
But \emph{how can a data analyst further explain these sample importance scores---why a sample receives a certain importance score?} Is a sample ``low-quality'' because it contains several ``moderate low quality'' features or an ``exceptionally low quality'' feature?

\end{example}

Answering the above questions calls for a quantitative understanding of how each \emph{block} in the data matrix (e.g. a sub-matrix as in Ex.~\ref{Ex1} or a single entry as in Ex.~\ref{Ex2} and \ref{Ex3}) contributes to the outcome of learning. 

\noindent
\textbf{Technical Challenges.} The problem of block valuation requires rethinking about fundamental aspects of data valuation. 
Existing data valuation theory consists of two basic modules at a conceptual level: (1) \emph{Counterfactual Analysis}, where one calculates how the utility of a subset of data sources would change after the source to be valued is removed; and (2) \emph{Fair Attribution}, where a data source is valued based on a weighted average of its marginal utilities for different subsets and the weights are set for the value to satisfy certain \emph{fairness} properties. 
The fairness notion considered by the past valuation schemes requires that permuting the order of different data sources does not change their value.

For horizontal and vertical valuation, the counterfactual can be simply calculated by taking the difference between the model performance trained on a subset of columns or rows and the performance with one column or row being removed. 
However, it is unclear how to calculate the counterfactual when a block is excluded because the remaining data matrix could be incomplete. Besides, the fairness notion of existing data value notions is no longer appropriate in the context of block valuation. 
As a concrete example to illustrate this point, consider Figure~\ref{fig:incomplete}(c) and suppose the two blocks on the left provide temperature measurements as features and the ones on the right are humidity measurements. 
In this case, one should not expect the value to be unchanged when two blocks with different physical meanings (e.g., yellow and pink) are swapped.

\noindent
\textbf{Contributions.} This paper presents the first focused study on data valuation without assuming shared feature space or sample space. 
Toward that end, we make the following contributions. 
\vspace{-1em}
\begin{itemize}
\itemsep-0.3em 
    \item We present an approach that enables evaluation of the marginal contribution of a block within the data matrix to any other block with non-overlapping sample and feature spaces.
    \item We abstract the block valuation problem into a two-dimensional (2D) cooperative game, where the utility function is invariant to column permutations and row permutations but not to any arbitrary entry permutations.
    \item We propose axioms that a proper valuation scheme should satisfy in the 2D game and show that the axioms lead to a unique representation of the value assignment (referred to as $\alg$). 
    Particularly, this representation is a natural generalization of the Shapley value~\cite{shapley1997value}---a celebrated value attribution scheme widely used in data valuation among other applications.
    \item We demonstrate that $\alg$ enables new applications, including selecting useful data fragments, providing interpretation for sample-wise data values, and fine-grained data issue diagnosis.
\end{itemize}

%% file: tex/related_work.tex
\vspace{-0.5em}
\section{Background and Related Work}\label{sec:related_work}

In a typical setting, a set of data sources are used to learn an ML model, which achieves a certain performance score. 
The goal of data valuation is to quantify the contribution of each data source toward achieving the performance score. The definition of a data source depends on the context in which the data valuation results are utilized. 
For instance, when using data valuation to interpret how the global behavior of the ML model depends on individual samples or individual features, a sample or a feature in the training data is regarded as a data source; when using data valuation to inform the reward design for data sharing, the collection of all samples or all features contributed by the same entities is regarded as a data source. 

Formally, let $N=\{1,\ldots,n\}$ denotes the index set of $n$ training data sources. A data valuation scheme assigns a score to each training data source in a way that reflects their contribution.
These scores are referred to as \textit{data values}. 
To analyze a source's ``contribution'', we define a \textit{utility function} $U: 2^N \rightarrow \R$, which maps any subset of the data sources to a score indicating the usefulness of the subset. 
$2^N$ represents the power set of $N$, i.e., the set of all subsets of $N$, including the empty set and $N$ itself. 
For the classification task, a common choice for $U$ is the performance of a model trained on the input subset, i.e., $U(S) = \text{acc}(\mathcal{A}(S))$, where $\mathcal{A}$ is a learning algorithm that takes a set $S\subseteq N$ of sources as input and returns a model, and $ \text{acc}$ is a metric function to evaluate the performance of a given model, e.g., the accuracy of a model on a hold-out validation set. 

Past research has proposed various ways to characterize data values given the utility function, among which the Shapley value is arguably the most widely used scheme for data valuation. 
The Shapley value is defined as 
\vspace{-0.5em}
{\small
\begin{equation}
\resizebox{.98\hsize}{!}{
$\psi_i^{1d}\left(U\right) \label{eqn:1d_shap} := \frac{1}{n} \sum_{k=1}^{n} \binom{n-1}{k-1}^{-1}  \sum_{\substack{S \subseteq N \setminus i\\ |S|=k-1}} \left[ U(S \cup i) - U(S) \right].$
}
\end{equation}
}

To differentiate from the proposed work, we will refer to the Shapley value defined in Eq.~(\ref{eqn:1d_shap}) as \oned. 
$\oned$ is popular due to its unique satisfaction of the following four axioms~\cite{shapley1953value}:
\begin{packeditemize}
    \item Dummy: if $U\left(S \cup i\right)=U(S)+c$ for any $S \subseteq N \setminus i$ and some $c \in \mathbb{R}$, then $\psi_i^{1d}\left(U\right)=c$.
    \item Symmetry: let $\pi:N\rightarrow N$ be any permutation of $N$ and $\pi U(S):= U(\pi(S))$, then $\psi_{\pi(i)}^{1d}(\pi U)=\psi_i^{1d}(U)$.
    \item Linearity: For utility functions $U_1, U_2$ and any $\alpha_1, \alpha_2 \in \R$, $\psi_i^{1d} \left(\alpha_{1} U_{1}+\alpha_{2} U_{2}\right)=\alpha_{1} \psi_i^{1d}\left(U_{1}\right)+$ $\alpha_{2} \psi_i^{1d}\left(U_{2}\right)$.
    \item Efficiency: for every $U, \sum_{i \in N} \psi_i^{1d}( U)=U(N)$.
\end{packeditemize}
The symmetry axiom embodies fairness. 
In particular, $\pi U$ arises upon the reindexing of data sources $1,\ldots,n$ with the indices $\pi(1),\ldots,\pi(n)$; the symmetry axiom states that the evaluation of a particular position should not depend on the indices of the data sources. 

Although the Shapley value was justified through these axioms in prior literature, the necessity of each axiom depends on the actual use case of data valuation results. 
Recent literature has studied new data value notions obtained by relaxing some of the aforementioned axioms and enabled improvements in terms of accuracy of bad data identification~\cite{kwon2022beta}, robustness to learning stochasticity~\cite{wang2023banzhaf,wu2022robust}, and computational efficiency~\cite{yan2021if}. 
For instance, relaxing the efficiency axiom gives rise to semi-values~\cite{kwon2022beta,wang2023banzhaf}; relaxing the linearity axiom gives rise to least cores~\cite{yan2021if}. 
This paper will focus on generalizing $\oned$ to block valuation. As we will expound on later, $\oned$ faces two limitations to serve a reasonable notion for block-wise values. 
Note that $\oned$ and the aforementioned relaxed notions share a similar structure: all of them are based on the marginal utility of a data source.
Hence, our effort to generalize the $\oned$ to new settings can be adapted to other more relaxed notions.




Another line of related work focuses on developing efficient algorithms for data valuation via Monte Carlo methods~\cite{jia2019towards,lin2022measuring}, via surrogate utility functions such as $K$-nearest-neighbors~\cite{jia2019efficient}, neural tangent kernels~\cite{wu2022davinz}, and distributional distance measures~\cite{hoang2023lava,tay2022incentivizing}, and via reinforcement learning~\cite{yoon2020data}. 
These ideas can also benefit the efficient computation of the proposed $\alg$. 
As a concrete example, this paper builds upon Monte Carlo simulation and surrogate model approaches to improve the efficiency of $\alg$.

Beyond data valuation, $\oned$ has been extensively used to gain feature-based interpretability for black-box models locally and globally. The local interpretability methods~\cite{lundberg2017unified,strumbelj2010efficient} focus on analyzing the relative importance of features for each input separately; therefore, the importance scores of features across different samples are not comparable. 
By contrast, our work allows the comparison of feature importance across different samples. 
The global interpretability methods~\cite{covert2020understanding}, on the other hand, explain the model’s behavior across the entire dataset. 
In the context of this paper, we consider them vertical data valuation. Compared to global interpretability methods, our work provides a more fine-grained valuation by associating each entry of the feature with an importance score. 
Our work improves the interpretability of the global feature importance score in the sense that it reveals the individual sample's contribution to the importance of a feature.

%% file: tex/2d_shapley.tex
\section{How to Value a Block?}\label{sec:2d_shapley}

\newcommand{\sn}{\mathbb{S}_N} 
\newcommand{\fm}{\mathbb{F}_M} 

This section starts with formulating the block valuation problem. 
Then, we will discuss the challenges of using $\oned$ to tackle the block valuation problem in terms of both counterfactual analysis and fair attribution. 
At last, we will present our proposed framework for solving the block valuation problem.

\subsection{Problem Formulation}


Let $N = \{1,2,\cdots, n \}$ and $M = \{1,2,\dots, m\}$, indexing $n$ disjoint collection of samples and $m$ disjoint collection of features contributed by $nm$ sources (or blocks). 
Each data source can be labeled by $(i,j)$ for $i\in N$ and $j\in M$, where we call $i$ the sample-wise index and $j$ the feature-wise index. To measure the contribution of a data source, we need to define a utility function, which measures the usefulness of a subset of data sources. 
The utility function $h(S,F)$ takes in two separate sets $S\subseteq N$ and $F\subseteq M$ as the variables and returns a real-valued score indicating the utility of $\{(i,j)\}_{i\in S,j\in F}$. Note that this paper focuses on valuing the relative importance of feature blocks; that is, we assume that each data contributor provides a block of features and then the aggregation of features will be annotated by a separate entity (e.g., a data labeling company) that does not share the profit generated from joint training. More formally, we define the utility function as follows:
\vspace{-0.5em}
\begin{align}
    &h(S,F):= \text{Performance of the model trained on the } \nonumber\\&\text{feature blocks }\{(i,j)\}_{i\in S,j\in F} \text{ after annotation.} \nonumber
\end{align}
\vspace{-1em}

One can potentially generalize our framework to jointly value feature and label blocks by redefining the utility function to be non-zero only when feature and label are both included in the input block, like~\cite{jia2019efficient,yona2021s}, but an in-depth investigation is deferred to future work.



The benefit of this utility function definition is two-fold. 
\emph{First}, its two-dimensional index always corresponds to a data fragment with the same feature space for all samples inside. 
As a result, one can calculate the utility in a straightforward manner by training on the matrix and evaluating the corresponding performance. This is an essential advantage over the one-dimensional index utilized by $\oned$, as will be exemplified later. 
\emph{Second}, created this way, the utility function is invariant to permutations of sample-wise indices in $S$ for any given $F$ and permutations of feature-wise indices in $F$ for any given $S$, but not to permutations of the sample-wise and feature-wise indices combined. 
This is a desirable property as for many data types in ML, such as tabular data, one would expect that swapping samples or swapping features~\footnote{Swapping features in an image dataset may lead to the loss of certain local information. However, it is rare that different pixel positions of an image dataset are contributed by different entities. So we will not consider this case.} does not change the model performance, yet swapping any two entries in the matrix may lead to arbitrary errors and thus alter the model performance significantly. 

Our goal is to assign a score to each block in $\{(i,j)\}_{i\in N,j\in M}$ that measures its contribution to the outcome of joint learning $h(N,M)$.

\vspace{-0.5em}
\subsection{A Naive Baseline: $\oned$} 

One idea to tackle the block valuation problem is to flatten the indices of blocks into one dimension and leverage $\oned$ to value each block. Specifically, we can reindex $\{(i,j)\}_{i\in N,j\in M}$ by $T=\{1,\ldots,nm\}$. 
Note that this step discards the structural information contained in the two-dimensional indices. 
Then, one can utilize Eq.~(\ref{eqn:1d_shap}) to value each $i\in T$.

The second step of applying Eq.~(\ref{eqn:1d_shap}) requires calculating $U(S\cup i)-U(S)$ for any $S\subseteq T\setminus i$. 
Both $S$ and $S\cup i$ could correspond to a data fragment with samples differing in their feature space (see example in Figure~\ref{fig:1d_shapley_shuffle}); nevertheless, how to evaluate the utility of such a fragment is unclear. 
An ad hoc way of addressing this problem is to perform missing value imputation, e.g., filling out the missing values of a feature using the average of the feature values present. 

In addition to the difficulty of evaluating the counterfactual, the symmetry axiom satisfied by $\oned$ no longer has the correct fairness interpretation when the input indices are flattened from 2D ones. 
In that case, $1,\ldots,nm$, carry specific meanings entailed by the original 2D structure; e.g., some indices might correspond to temperature features, and others might correspond to humidity. 
Hence, the symmetry axiom that requires unchanged data values after permuting the data sources' indices is not sensible and necessary, as the permutation might map the content of a data source from one meaning to an entirely different one.

We will use $\oned$ with missing value imputation as a baseline for our proposed approach. 
This simple baseline is still a useful benchmark to assess the extra (non-trivial) gains in different application scenarios that our approach can attain.





\begin{figure}[htb]
\vspace{-0.5em}
\begin{center}
  \includegraphics[width=0.46\textwidth]{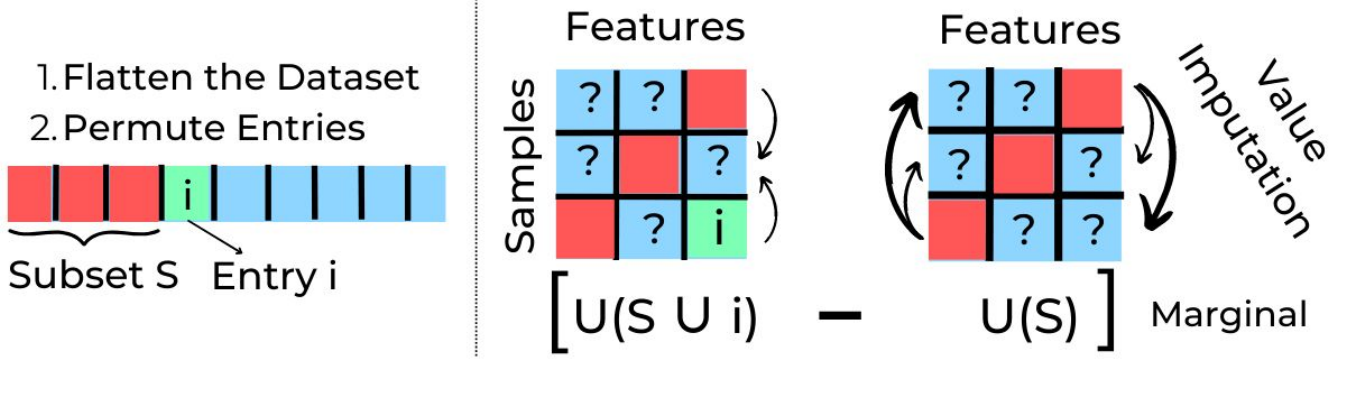}
    \vspace{-1em}
  \caption{A visualization of $\oned$ marginal contribution applied to sample-feature valuation.}~\label{fig:1d_shapley_shuffle}
  \end{center}
  \vspace{-2em}
\end{figure}




\subsection{Our Approach: $\alg$}
Here, we will describe $\alg$ as a principled framework for valuing data blocks. 
We will emphasize how $\alg$ overcomes the challenges of the $\oned$ baseline in terms of (1) calculating the counterfactual, (2) framing the correct fairness principles, and then derive the representation of the data values based on the new counterfactual analysis and principle. 
At last, we will show efficient algorithms to compute $\alg$.

\subsubsection{Two-Dimensional Counterfactual Analysis}


Given a two-dimensional utility function $h(\cdot, \cdot)$, we will define the marginal contribution of a block $(i,j)$ to the collection of blocks $\{(i,j)\}_{i\in S,j\in F}$ as 
{\small
    \begin{align}
	    M_h^{i,j}(S, F):=&h(S\cup i, F\cup j)+h(S, F) \nonumber \\
	    -&h(S\cup i, F)-h(S, F\cup j).\label{eq:margin}
	\end{align}

\vspace{-0.0em}
}
\setlength{\intextsep}{1em}%
\setlength{\columnsep}{1em}%

The rationality of the definition of $M_h^{i,j}(S, F)$ can be shown by Figure \ref{fig:2d_shapley}.
The area corresponding to $h(S\cup i, F\cup j)$ can be viewed as the area $(S\cup i, F\cup j)$, which subtracts these two areas of $(S\cup i, F)$ and $(S, F\cup j)$, plus the $(S, F)$ area that is subtracted twice, the remaining area is shown in Figure \ref{fig:2d_shapley} as ``marginal'', which corresponds to the marginal influence of the block $(i,j)$.  
\begin{wrapfigure}{r}{0.36\columnwidth}
 \vspace{-1.5em}
  \begin{center}
    \includegraphics[width=0.38\columnwidth]{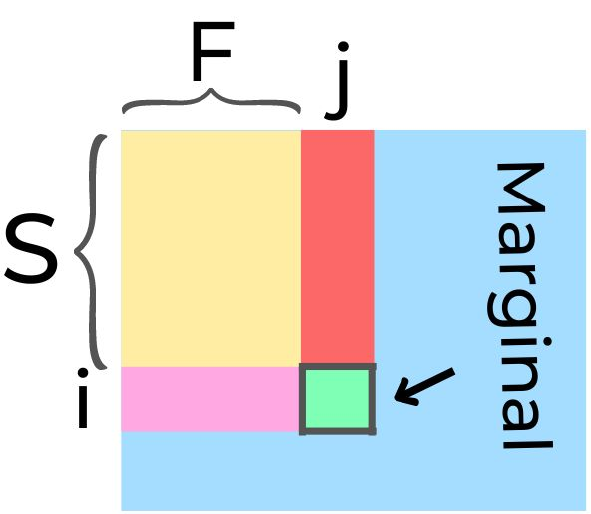}
  \end{center}
  \vspace{-1.0em}
  \caption{Removal process and marginal influence of $(i,j)$.}~\label{fig:2d_shapley}
  \vspace{-2em}
\end{wrapfigure} 
The unique advantage is that each individual utility is well-defined as it takes as input a collection of blocks within which the samples all share same feature space.

\subsubsection{Axioms for Block Valuation}

We start by redefining ``dummy'' for block valuation, where the underlying utility function is 2D.

\begin{definition}(2D-Dummy)
	    We call a block $(i,j)$ a 2D-dummy under utility function $h$ if for all $S\subseteq N\backslash i$ and $F\subseteq M\backslash j$, 
	    \begin{equation}\label{def:dummy}
	         M_{h}^{i,j}(S,F)=c, c\in \mathbb{R}.
	    \end{equation}
\end{definition}
$2D$-dummy implies the canonical (one-dimensional) dummy mentioned in Section~\ref{sec:related_work}. Specifically, if sample $i$ is a sample dummy which satisfies $h(S\cup i, F)=h(S,F)+c_1$ and $h(S\cup i, F\cup j)=h(S,F\cup j)+c_2$ for $S\subseteq N\backslash i,F\subseteq M\backslash j$ like the dummy defined in $\oned$, then Eq.~(\ref{def:dummy}) is satisfied with $c:=c_2-c_1$, and similarly, if feature $j$ is a feature dummy which satisfies $h(S,F\cup j)=h(S,F)+c_1'$ and $h(S\cup i,F\cup j)=h(S\cup i,F)+c_2'$ for $S\subseteq N\backslash i,F\subseteq M\backslash j$, then Eq.~(\ref{def:dummy}) is also satisfied with $c:=c_2'-c_1'$. However,  Eq.~(\ref{def:dummy}) can not imply sample $i$ is a sample dummy or feature $j$ is a feature dummy. 

We first define the utility function set $G$ which contains all possible utility functions, and define a value function $\psi: G \rightarrow \mathbb{R}^{n\times m}$ and denote the value of block $(i,j)$ as $\psi_{ij}(h)$ which is the $ij$th element in matrix $\psi(h)$. In order to build an equatable evaluation system, we provide the following axioms. 

\begin{axiom}(2D-Linearity)\label{axiom: linear}
	For any two utility functions $h_1, h_2\in G$ and any $\beta_1, \beta_2\in \mathbb{R}$,
	\begin{equation}
	  \psi_{ij}(\beta_1h_1+\beta_2h_2)=\beta_1\psi_{ij}(h_1)+\beta_2\psi_{ij}(h_2).
	\end{equation}
\end{axiom}
\begin{axiom}(2D-Dummy)\label{axiom: dummy}
    If the block $(i,j)$ is a dummy of $h$ which satisfies Eq.~(\ref{def:dummy}), then $\psi_{ij}(h)=c$.
\end{axiom}
\begin{axiom}(2D-Symmetry)\label{axiom: symmetry}
    Let $\pi_1: N\rightarrow N$ and $\pi_2: M\rightarrow M$ be two permutations, then:
    \begin{equation}
        \psi_{\pi_1(i)\pi_2(j)}[(\pi_1\pi_2)h]=\psi_{ij}(h),
    \end{equation}
    where for all $S\subseteq N, F\subseteq M$,
    \begin{equation}
        [(\pi_1\pi_2)h](S, F):=[(\pi_2\pi_1)h](S, F):=h(\pi_1(S), \pi_2(F)).
    \end{equation}
\end{axiom}
\begin{axiom}(2D-Efficiency)\label{axiom: efficiency}
    For every utility function $h\in G$, 
    {\small
    \begin{equation}
         \sum_{\substack{i\in N\\j\in M}}\psi_{ij}(h)=h(N, M).
    \end{equation}
    }
\end{axiom}

Let us discuss the rationality of the four axioms.

   The \underline{2D-linearity axiom} is inherited from $\oned$, which implies that the value of the $(i,j)$-th block under the sum of two ML performance measures is the sum of the value under each performance measure. 
    
    The \underline{2D-dummy axiom} can be interpreted by taking $c=0$. If a block has no contribution to the ML task, no matter what the situation (i.e., for any $S\subseteq N\backslash i$ and $F\subseteq M\backslash j$), then its value is zero. 
    
   

In the \underline{2D-symmetry axiom}, the rows and columns are permuted independently. As a result, the entries from the same feature will always remain in the same column. The axiom state that such permutations would not change the value for individual data blocks, which is what we would expect in many ML applications.
    In Appendix \ref{appendix: proof for equivalence}, we proved that Axiom \ref{axiom: symmetry} is implied by explanation here.
    
   The \underline{2D-efficiency axiom} is inherited from $\oned$, requiring that the sum of the values of all the data blocks equals the performance of the whole data set. 

Based on the axioms, we provide a definition:

\begin{definition}
    The value $\psi_{ij}(h)$ with respect to the utility function $h$ is a \textbf{two-dimensional Shapley value} ($\alg$ for short) if $\psi_{ij}$ satisfies the 2d-linearity, 2d-dummy, 2d-symmetry  and 2d-efficiency axioms, denoting as $\psi_{ij}^{2d}$.
\end{definition}
$\alg$ can be seen as the two-dimensional extension of Shapley values, which inherits its advantage with a natural adaptation of the dummy and symmetry axiom to the two-dimensional utility function scenario.

\subsubsection{Representation Theory}
We will show that there exists an analytic and unique solution for \alg. 

\begin{theorem}(Representation Theory of $\alg$)\label{thm: Representation of 2D Shapley-value}
The $\psi^{2d}_{ij}$ has a unique solution: 
\begin{equation}\label{eqn: 2D Shapley}
     \psi_{ij}^{2d}=\frac{1}{nm}\sum_{s=1}^n\sum_{f=1}^m\Delta_{sf},
\end{equation}
where $i\in N$, $j\in M$,
 \begin{equation}\label{Eq:marginal_contribution}
         \begin{aligned}
              \Delta_{sf} = \frac{1}{\tbinom{n-1}{s-1}\tbinom{m-1}{f-1}}\sum_{(S, F)\in D_{sf}^{ij}}&M_h^{i,j}(S, F),
         \end{aligned}
     \end{equation}
$
	    D_{sf}^{ij} = \{(S, F): S\subseteq N\backslash i, F\subseteq M\backslash j, |S|=s-1, |F|=f-1\},
$   
and $M_h^{i,j}(S, F)$ defined in Eq.~(\ref{eq:margin}).
\end{theorem}

Theorem \ref{thm: Representation of 2D Shapley-value} indicates that $\psi_{ij}^{2d}$ is a weighted average of the two-dimensional counterfactual in Eq. (\ref{eq:margin}).
Theorem \ref{thm: Representation of 2D Shapley-value} is referred to as the representation theory of $\alg$, because the proof procedure shows that $\psi_{ij}^{2d}$ has a basis expansion formulation (see Eq. (\ref{eq:basis_expansion}) in Appendix \ref{sec:appendix_representation_2dshapley}).
To show the basis expansion, a series of basic utility functions in $G$ needs to be defined (e.g., Eq. (\ref{eq:one_basis})).
Compared with the representation theory of $\oned$ by \citet{roth1988shapley}, one technical challenge is to define the basis and basic utility functions for the 2D case to handle the 2D counterfactual.
Furthermore, the proof of the uniqueness of $\alg$ has to solve a complex high-dimensional linear system (see Eq.~(\ref{linear recursive condition1}) in Appendix \ref{sec:appendix_representation_2dshapley}). Our proof incorporates new techniques, unseen in the classic proof of $\oned$, to deal with these unique technical challenges arising in the 2D context.




Moreover, the representation theory also implies that $\alg$ can be reduced to $\oned$. 
The following corollary shows that summing up the block values over all rows gives $\oned$ of features, and summing up the block values over all columns gives $\oned$ of samples. 
Corollary \ref{cor: 2d to 1d} does not only indicate that the $\alg$ is a natural generalization of $\oned$, but also is
useful for discussing the experimental results of how 2D values can explain 1D values (see Subsection \ref{subsec: cell_valu}). 
\begin{corollary}\label{cor: 2d to 1d}
For any $h\in G$, let $\psi^{1d}_{i\cdot}(h):=\sum_{j\in M}\psi^{2d}_{ij}(h)$ and $\psi^{1d}_{\cdot j}(h):=\sum_{i\in N}\psi^{2d}_{ij}(h)$, then 
{\small
\begin{equation}
    \psi^{1d}_{i\cdot}(h)=\frac{1}{n}\sum_{\substack{S\subseteq N\backslash i\\|S|=s}}\frac{1}{\tbinom{n-1}{s}}[h(S\cup i,M)-h(S,M)],
\end{equation}
}
and
{\small
\begin{equation}
    \psi^{1d}_{\cdot j}(h)=\frac{1}{m}\sum_{\substack{F\subseteq M\backslash j\\|F|=f}}\frac{1}{\tbinom{m-1}{f}}[h(N,F\cup j)-h(N, F)],
\end{equation}
}
which are in the form of \oned.
\end{corollary}

Finally, having the analytical expression Eq. (\ref{eqn: 2D Shapley}) of $\alg$ at hand will provide us with great convenience in designing efficient algorithms.

\subsubsection{Efficient Algorithm}


The computational complexity of exactly calculating $\alg$ is exponential in $mn$ due to the summation over all possible rows and columns.
To overcome this challenge, we develop a Monte Carlo approach to approximating $\alg$. The key idea is that $\alg$ can be rewritten as an expectation of the marginal contribution of $(i,j)$ to the blocks indexed by row indices before $i$ and column indices before $j$ over random permutations of rows and columns. 
As a result, we can approximate $\alg$ by taking an average over randomly sampled rows and columns. 
We also design the algorithm in ways that can reuse utility function evaluations across different permutations, which gives rise to significant efficiency gains. The full details of the algorithm design are provided in Appendix~\ref{sec:algo}, and the pseudo-code is shown in Algorithm~\ref{alg:2d-mc}. 



Evaluating the utility function requires retraining a model. For small-scale datasets, it might be possible to evaluate the utility function within a reasonable time multiple times, but for large-scale datasets, even evaluating it once might require days to finish. 
This would deem our method impractical for any applications. 
Nonetheless, we can even obviate all model training to compute our values when using $K$-nearest-neighbor (KNN) as a surrogate model. KNN-surrogate-based data valuation has shown great computational advantage while providing effective data quality identification~\cite{jia2019efficient}. 
In this work, we leverage a similar idea to reduce the computational complexity of $\alg$ for large models. First, let us observe from Eq.~(\ref{eqn: 2D Shapley}) and Corollary.~\ref{cor:eff_2d} that after rearranging inner terms, we have:

\vspace{-1em}
{\small
\begin{align} \label{eq:2dperm-arrange}
    &\psi_{ij}^{2d}=\frac{1}{n!m!}\sum_{\substack{\pi_1 \in \Pi(N) \\ \pi_2 \in \Pi(M) }} 
    \left[h(P_{i}^{\pi_1} \cup i, P_{j}^{\pi_2} \cup j) -\right.  \\
    & \left.h(P_{i}^{\pi_1}, P_{j}^{\pi_2}\cup j )\right] - \left[h(P_{i}^{\pi_1}\cup i, P_{j}^{\pi_2} )- h(P_{i}^{\pi_1}, P_{j}^{\pi_2}) \right],\nonumber
\end{align}
}

where $\Pi(X)$ is a set of all permutations of $X$, $\pi \in \Pi(X)$ is a permutation of $X$, and $P^{\pi}_i$ is a set of elements preceding $i$ in $\pi$.
The expression in the first bracket is the 1D marginal contribution of sample $i$ and is valid since both utilities are trained on same features, $P_j^{\pi_2} \cup j$. 
Similarly, the second bracket also represents a valid 1D marginal contribution of the sample $i$ but with features $P_j^{\pi_2}$. From this observation, we can apply the results of $\oned$ value approximated with nearest neighbors, $\phi^\text{KNN}$, defined recursively in Theorem 1~\cite{jia2019efficient}, and the $\alg$ under KNN surrogates can be therefore expressed as

\vspace{-1em}
{\small
\begin{equation} \label{eq:2dknn-sv}
\psi_{ij}^\text{2d-KNN} =\frac{1}{m!}\sum_{\substack{ \pi_2 \in \Pi(M) }}  [\phi^\text{KNN}(i, P_{j}^{\pi_2} \cup j) - \phi^\text{KNN}(i, P_{j}^{\pi_2})].\nonumber
\end{equation}
}
\vspace{-1em}

This new formulation is efficient as it requires no more model training and removes the summing over all possible permutations of samples. We can further approximate the sum over all possible permutations over features with the average over sampled permutations.
Our final complexity becomes $\mathcal{O}(PT|M||N|^2log|N|)$, where $P$ is the number of sampled feature permutations, $T$ is the number of test points used for evaluating model performance, and $|N|,|M|$ are the cardinality of $N$ and $M$ respectively, and the pseudo-code for the overall KNN-based approximation is provided in Algorithm~\ref{alg:2d-knn}.







%% file: tex/experiments.tex
\section{Experiments}~\label{sec:experiments}
\vspace{-0.5em}

\begin{figure*}[t]
\begin{center}
  \includegraphics[width=1\textwidth]{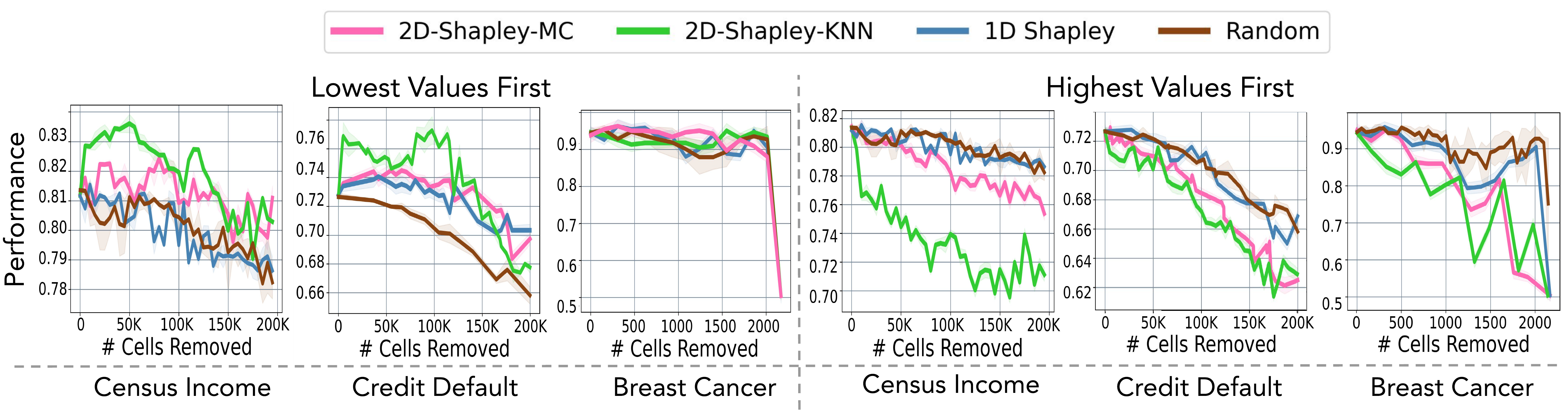}
  \vspace{-1em}
  \caption{Performance comparison between $\alg$ and baselines on various use cases.}~\label{fig:2d_vals}
  \end{center}
  \vspace{-1em}
\end{figure*}
\vspace{-1em}

This section covers the two general application scenarios of $\alg$. (1) \emph{Cell valuation}, where each cell in the training data matrix is considered a data source and receives a score indicating its contribution to a learning task performed on the matrix. We mainly demonstrate this application scenario's benefits in fine-grained data debugging and interpreting canonical sample-wise or feature-wise data values. (2) \emph{Sub-matrix valuation}, where a sub-matrix containing multiple cells is considered a data source and receives a joint score. This scenario is closely related to data marketplaces, where each entity provides a dataset that appears as a submatrix in the aggregated data. Details about datasets, models, implementations, and ablation studies on a budget of inserted outliers are provided in Appendix~\ref{sec:res}.





\subsection{Cell Valuation}\label{subsec: cell_valu} 








\paragraph{Sanity check of cell-wise values.} We first check whether the cell-wise values produced by our method make sense via the data removal experiments commonly used in the data valuation literature. 
Specifically, we would expect that removing the cells with the highest values from the training set leads to the most significant performance degradation; conversely, removing the cells with the lowest values should barely affect the model performance. 
To evaluate the model performance after removal, we ``remove'' a cell by refilling its content with the average of all other cells on the same feature column. 
In the previous section, we present two algorithms to calculate $\alg$. We will label the values obtained from the Monte Carlo-based method as $\algmc$ and the ones from the KNN-surrogate-based method as $\algknn$. $\oned$ and random removal are used as our baselines. 
In particular, $\oned$ is estimated by the permutation sampling described in~\cite{jia2019towards}. 
For each baseline, we remove a number of cells at a time based on their sample-feature value ranking in either descending or ascending order; then, we train a model on the reduced dataset and evaluate the model performance. 

As shown in Figure~\ref{fig:2d_vals}, when removing cells in ascending value order, $\alg$ can not only maintain the model performance but also improve it by at least $2\%$ for Census, Credit, and Breast Cancer datasets, whereas $\oned$ dips the model performance earlier than $\alg$ in all three datasets. 
Considering removal from the highest valued cells, we observe that $\alg$ can effectively detect contributing cells, and removing these cells causes the model performance to drop quickly. 
By contrast, removing cells according to $\oned$ is close to random removal. These results indicate that $\alg$ is more effective than $\oned$ at recognizing the contribution of cells and can better inform strategic data harnessing in ML.
\vspace{-1 em}
\paragraph{Fine-Grained Outlier Localization.} Existing horizontal data valuation methods have demonstrated promising results in detecting abnormal samples~\cite{ghorbani2019data,kwon2022beta,wang2023banzhaf} by finding lowest-valued samples. 
However, it is rarely the case that every cell in the sample is abnormal. For instance, a type of error in the Census data is ``198x$\rightarrow$189x'', where the years of birth are wrongly specified; this error could appear on a single feature column and, at the same, only affects partial samples (or users) born in 198x. 
Existing horizontal valuation remains limited in localizing these erroneous entries. 

To demonstrate the potential of $\alg$ in fine-grained entry-wise outlier detection, we first inject outlier cells into the clean dataset, Breast Cancer Dataset. 
Following a recent outlier generation technique in~\cite{du2022vos}, we inject low-probability-density values into the dataset as outlier cells. 
We explain the outlier injection method in detail in Appendix~\ref{sec:outlier}.
We randomly place outlier cells in $~2\%$ (50 of total cells). 
Afterward, we compute $\algknn$ for each cell in the dataset with inserted outliers, which are shown in Figure~\ref{fig:bc_2d_out_heatmap}.  
Since we expect outliers not to be helpful for the model performance, the values for outlier cells should be low. 
Therefore, we sort the $\alg$ cell values in ascending order and prioritize human inspection towards the ones with the lowest values. 
We show the detection rate of the inserted outliers in Figure~\ref{fig:census_breast_outlier_detection}A). 
As we can see, with $\alg$ values, we can detect $90\%$ of inserted outliers within the first $5\%$ of all cells. 
By contrast, based on the values produced by $\oned$, one would need over $90\%$ of cell inspection to screen out all the outlier cells.



\begin{figure*}
\centering
\begin{minipage}{.38\textwidth}
  \centering
  \includegraphics[width=1\linewidth]{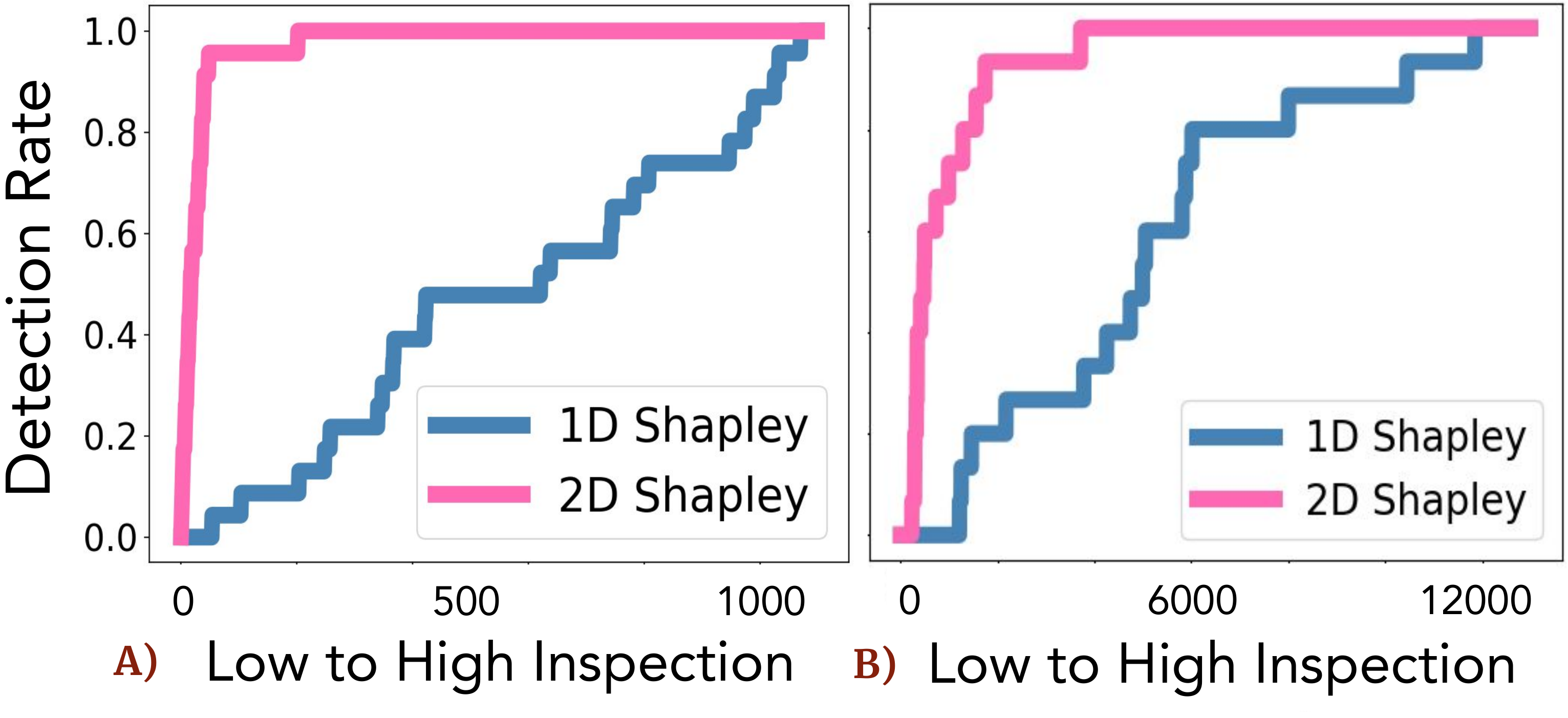}
  \captionof{figure}{A) Detection of the inserted outliers in the Breast Cancer dataset. B) Detection of the inserted outliers in the Age category of the Census
  dataset. }
  \label{fig:census_breast_outlier_detection}
\end{minipage}
\hfill
\begin{minipage}{.25\textwidth}
  \centering
  \includegraphics[width=1\linewidth]{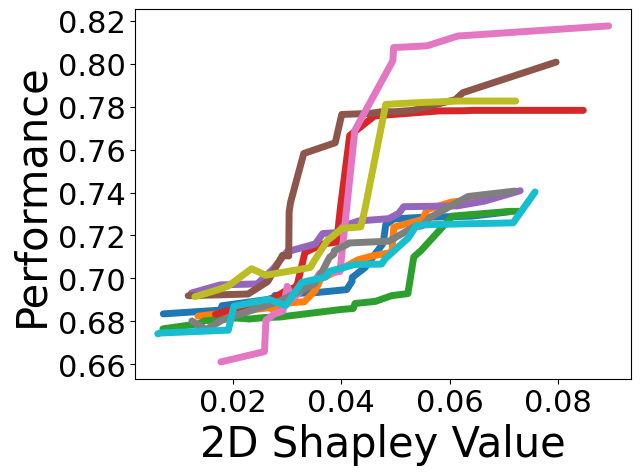}
  \captionof{figure}{2D Shapley vs Model Performance on various dataset splits.}\label{fig:2d_vs_perf}
\end{minipage}
\hfill
\begin{minipage}{.32\textwidth}
  \centering
  \includegraphics[width=1\linewidth, height=100pt]{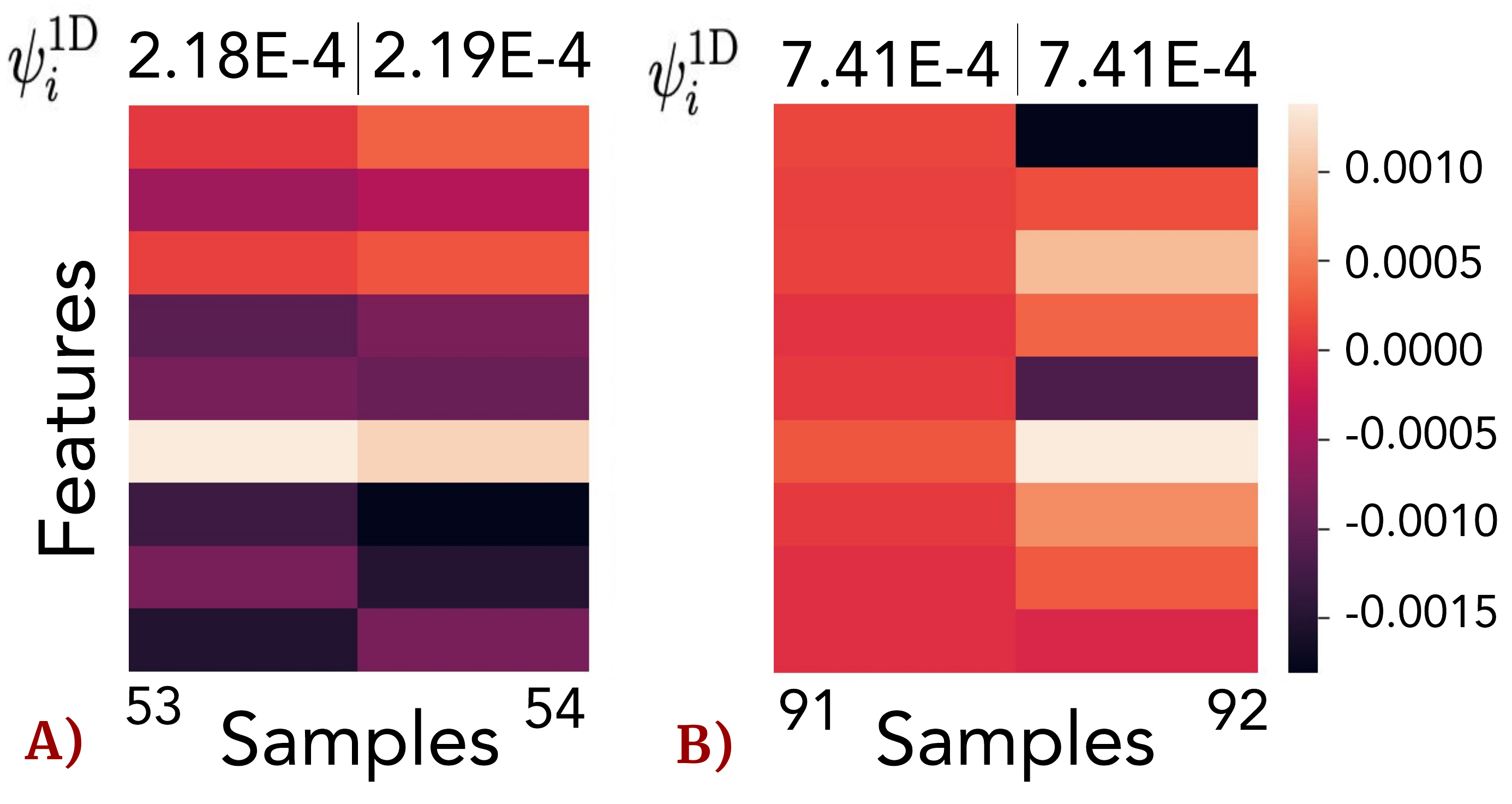}
  \captionof{figure}{Cell values of samples with similar 1D values in Breast Cancer dataset.}\label{fig:values_understand}
\end{minipage}
\vspace{0em}
\end{figure*}








We further examine a practical case of outliers caused by human errors, where 
the cells have been incorrectly typed, e.g., ``18'' became ``81''. 
In the Census dataset, for the feature ``Age'', we randomly swap 15 cells between ``17'' and ``71'', ``18'' and ``81'', ``19'' and ``91''. 
Similarly, we sort the values of all cells in the dataset in ascending order. 
As we observe in Figure~\ref{fig:census_breast_outlier_detection}B), detection with $\alg$ outperforms $\oned$. Particularly, with $\alg$ we can detect $80\%$ of added outliers with less than $1800$ inspected cells while $\oned$ requires $4$ times as many cells to achieve a comparable rate. 
The $\oned$ and $\alg$ heatmaps are provided in Appendix. 
The results above demonstrate the effectiveness of $\alg$ in locating outlier cells in a dataset.
\vspace{-1em}
\paragraph{Enabling Interpretation of 1D Valuation Results.} Apart from outlier detection, $\alg$ also brings new insights into horizontal sample valuation or vertical feature valuation, which is referred to as 1D valuation. For instance, 1D sample valuation produces an importance score for each sample, but we lack a deeper understanding of why a sample receives a certain value. Recall Corollary~\ref{cor: 2d to 1d} that the sum of $\alg$ over rows or columns gives 1D feature values and 1D sample values, respectively. Hence, $\alg$ allows one to interpret the 1D value of a sample by further breaking it down to contributions of different features in that sample. That is, $\alg$ gives insights into the relative importance of different features of a sample to the valuation result received by the sample. 
For example, in Figure~\ref{fig:values_understand}A), we observe that two samples have similar 1D values and their cell values are also close. However, in Figure~\ref{fig:values_understand}B), we observe a contrasting case, where although both samples have a close 1D value, their cell values are completely unrelated. \edit{More detailed results can be found in Appendix~\ref{sec:outlier}}.

\vspace{-0.5em}
\subsection{Sub-matrix Valuation} 







We turn to the application of $\alg$ to inform dataset pricing in the data marketplace.
$\alg$ enables a principled method to value fragmented data sources as illustrated in Figure~\ref{fig:incomplete}(c), where each source is a sub-matrix in the aggregated training data matrix. A reasonable measure of a source's value should reflect its usefulness for ML.
Hence, to verify the significance of the resulting values for sub-matrix valuation, we measure the model performance trained on a source and examine the correlation between its value and the performance. For this experiment, we use the Credit Dataset with sources contributing fragmented data and consider multiple random splits of the dataset.
The results are provided in Figure~\ref{fig:2d_vs_perf}, where each line corresponds to a different split of the aggregate data into individual sources. Figure~\ref{fig:2d_vs_perf}
shows that with the increasing model performance trained on the block, its corresponding $\alg$ block value also increases.

%% file: tex/conclusion.tex
\vspace{-0.5em}
\section{Conclusion}~\label{sec:conclusion}
\vspace{-1.5em}

This work aims to set the theoretical foundation for more realistic data valuation application scenarios. 
In particular, we investigate the block valuation problem and present $\alg$, a new data value notion that is suitable to solve this problem. $\alg$ empowers a range of new use cases, such as informing the pricing of fragmented data, strategic data selection on a fine-grained scale, and interpreting 1D valuation results.
Our work opens up many new venues for future investigation. First, we can immediately adapt our proof technique to prove a two-dimensional generalization of other typical data value notions~\cite{kwon2022beta,wang2023banzhaf}. Second, it is interesting to build upon our framework to evaluate irregular-shaped data sources~\cite{fang2019model} and incorporate label information for joint valuation in a principled way.






%% file: tex/acknowledge.tex
\section*{Acknowledgements}~\label{sec:Acknowledge}

Xiangyu Chang's work was partly supported by the National Natural Science Foundation for Outstanding Young Scholars of China under Grant 72122018 and partly by the Natural Science Foundation of Shaanxi Province under Grant 2021JC-01. Xi Chen would like to thank the support from NSF via the Grant IIS-1845444. RJ and the ReDS lab would like to thank the support from NSF via the Grant OAC-2239622.

%% file: tex/appendix.tex
\appendix{}


\begin{appendices}{}



\section{Proof of the fact that Axiom \ref{axiom: symmetry} is implied by its explanation}\label{appendix: proof for equivalence}
The explanation above is: for $i_1, i_2 \in N$, $j_1,j_2\in M$, if for any $S\subseteq N\backslash \{i_1, i_2\}$ and $F\subseteq M$, 
$h(S\cup i_1,F)=h(S\cup i_2,F)$, and for any $S\subseteq N$ and $F\subseteq M\backslash\{j_1, j_2\}$, $h(S,F\cup j_1)=h(S,F\cup j_2)$,
then $\psi_{i_1j_1}(h)=\psi_{i_2j_2}(h)$.

For the proof,  we prove in three steps that the explanation is equivalent to Axiom \ref{axiom: symmetry}. Note that we should assume Axiom \ref{axiom: linear}, \ref{axiom: dummy} and \ref{axiom: efficiency} already exist. For simplicity, we use the lowercase letter to denote the cardinality of a set, for example, $|S|=s$.

We want to prove the following proposition.
\begin{proposition}\label{prop: equivalence}
If Axiom \ref{axiom: linear}, \ref{axiom: dummy} and \ref{axiom: efficiency} exist, then Axiom \ref{axiom: symmetry} is equivalent to its explanation.
\end{proposition}
\begin{proof}
For the direction that Axiom \ref{axiom: symmetry} is implied by its explanation, we prove in three steps.
\begin{itemize}
    \item Step 1: 
    Define a utility function $h_{S,F}$:
    \begin{equation}\label{eq:one_basis}
	    h_{S,F}(W_1, W_2)=\left\{
	\begin{array}{lr}
		1, if \ S\subseteq W_1 , F\subseteq W_2.\\
		0, otherwise.
	\end{array}
	\right.
	\end{equation}
	For fixed $S\subseteq N$, $F\subseteq M$ and $i_1, i_2\in S$, $j_1,j_2\in F$ and for all $W_1\subseteq N\backslash\{i_1,i_2\},W_2\subseteq M\backslash\{j_1,j_2\}$, $M_{h_{S,F}}^{i_1,j_1}(W_1,W_2)=M_{h_{S,F}}^{i_2,j_2}(W_1,W_2)$. It leads to the conclusion that $\psi_{i_1j_1}(h_{S,F})=\psi_{i_2j_2}(h_{S,F})$ according to the explanation.
	
	For $i^*\notin S$, $j\in M$ (or $j^*\notin F$, $i\in N$) and $W_1\subseteq N\backslash i^*$, $W_2\subseteq M\backslash j$, ($W_1\subseteq N\backslash i$, $W_2\subseteq M\backslash j^*$,) $M_{h_{S,F}}^{i^*,j}(W_1,W_2)=0$. ($M_{h_{S,F}}^{i,j^*}(W_1,W_2)=0$.) 
	It leads to the conclusion that $\psi_{i^*j}(h_{S,F})=0$, $\forall j\in M$ ($\psi_{ij^*}(h_{S,F})=0$, $\forall i\in N$) according to Axiom \ref{axiom: dummy}.
	
	 In summary, we have conclusion that the values $\psi_{ij}$s are the same when $i\in S, j\in F$, and otherwise zero.
	According to Axiom \ref{axiom: efficiency}, 
	\begin{equation*}
	    1=h_{S,F}(N,M)=\sum_{\substack{i\in N\\j\in M}}\psi_{ij}(h_{S,F})=\sum_{\substack{i\in S\\j\in F}}\psi_{ij}(h_{S,F}).
	\end{equation*}
	then $\psi_{ij}(h_{S,F})=1/sf$, where $i\in S, j\in F$.
	
	\item Step 2: 
	We have to prove a lemma which shows another formation of a utility $h$ by using $h_{S,F}$ defined above.
	\begin{lemma}
	\begin{equation*}
	    h=\sum_{\substack{S\subseteq N\\F\subseteq M}}C_{S,F}(h)h_{S,F},
	\end{equation*}
	where $C_{S,F}(h)=\sum_{\substack{S'\subseteq S\\F'\subseteq F}}(-1)^{s+f-s'-f'}h(S',F')$.
	\end{lemma}
	\begin{proof}
	We can directly verify the lemma.
	\begin{align*}
	    h(W_1,W_2)&=\sum_{\substack{S\subseteq N\\F\subseteq M}}C_{S,F}(h)h_{S,F}(W_1,W_2)\\
	    &=\sum_{\substack{S\subseteq W_1\\F\subseteq W_2}}\sum_{\substack{S'\subseteq S\\F'\subseteq F}}(-1)^{s+f-s'-f'}h(S',F')\\
	    &=\sum_{\substack{S'\subseteq W_1\\F'\subseteq W_2}}\big[\sum_{s=s'}^{w_1}(-1)^{s-s'}\tbinom{w_1-s'}{s-s'}\sum_{f=f'}^{w_2}(-1)^{f-f'}\tbinom{w_2-f'}{f-f'}\big]h(S',F')\\
	    &=h(W_1,W_2).
	\end{align*}
	\end{proof}
	
	\item Step 3: 
	Combine the first two steps, and by Axiom \ref{axiom: linear},
	\begin{align*}
	    \psi_{ij}(h)&=\sum_{\substack{S\subseteq N\\F\subseteq M}}C_{S,F}(h)\psi_{ij}(h_{S,F})\\
	    &=\sum_{\substack{i\in S\subseteq N\\j\in F\subseteq M}}C_{S,F}(h)/sf.
	\end{align*}
	Let $\pi_1,\pi_2$ be two permutations on $N$ and $M$ respectively, then
	\begin{align*}
	    \psi_{\pi_1(i)\pi_2(j)}(\pi_1\pi_2h)&=\sum_{\substack{\pi_1(i)\in S\subseteq N\\\pi_2(j)\in F\subseteq M}}C_{S,F}(\pi_1\pi_2h)/sf\\
	    &=\sum_{\substack{i\in \pi_1(S)\subseteq N\\j\in \pi_2(F)\subseteq M}}C_{\pi_1(S),\pi_2(F)}(h)/sf\\
	    &=\psi_{ij}(h).
	\end{align*}
	
\end{itemize}

For another direction that Axiom \ref{axiom: symmetry} implies its explanation, since we already assume Axiom \ref{axiom: linear}, \ref{axiom: dummy}, \ref{axiom: symmetry} and \ref{axiom: efficiency} hold, then we have the formula of $\alg$, that is, Eq.~(\ref{Eq:2dShap_set_express}). Clearly, we can see the numerator is always the same for both $i_1j_1$ and $i_2j_2$ under the same $S$ and $F$, hence $\psi_{i_1j_1}(h)=\psi_{i_2j_2}(h)$.

\end{proof}

\section{Proof of the representation theory of $\alg$}\label{sec:appendix_representation_2dshapley}
In this section, we will justify the representation theory by a number of proposed lemmas.
The proof process is to add the axioms one by one and try to show what each axiom does for $\alg$.
We add linearity and dummy axioms first to get a sum of weighted marginals.
\begin{lemma}\label{lemma:linearity and dummy}
For any value $\psi_{ij}$ satisfying the 2d-linearity and 2d-dummy axioms (Axiom \ref{axiom: linear} and \ref{axiom: dummy}), we have that
\begin{align}
   &\psi_{ij}(h)=\sum_{S\subseteq N\backslash i}\sum_{F\subseteq M\backslash j}p^{ij}_{S,F}[h(S\cup i,F\cup j)+h(S,F)
    \nonumber 
    -h(S\cup i,F)-h(S,F\cup j)],
    where\ \sum_{S\subseteq N\backslash i}\sum_{F\subseteq M\backslash j}p^{ij}_{S,F}=1.
\end{align}
where $\sum_{S\subseteq N\backslash i}\sum_{F\subseteq M\backslash j}p^{ij}_{S,F}=1$.
\end{lemma}
\begin{proof}
For any $h\in G$, 
\begin{align}\label{eq:basis_expansion}
    h&=\sum_{\substack{S\subseteq N\\F\subseteq M}}h(S, F)W_{S, F},
\end{align}
where 
\begin{equation*}
	    W_{S, F}(W_1, W_2)=\left\{
	\begin{array}{lr}
		1, if \ W_1=S, W_2=F.\\
		0, otherwise.
	\end{array}
	\right.
	\end{equation*}
By the 2d-linearity axiom,
\begin{align*}
     \psi_{ij}(h)&=\sum_{\substack{S\subseteq N\\F\subseteq M}}h(S, F)\psi_{ij}(W_{S, F}).
\end{align*}
Now define another utility function $W'_{S, F}$:
	\begin{equation*}
	    W'_{S, F}(W_1, W_2)=\left\{
	\begin{array}{lr}
		1, if \ S\subseteq W_1, F=W_2.\\
		0, otherwise.
	\end{array}
	\right.
	\end{equation*}
For any $S\subseteq N\backslash i$ and $F\subseteq M\backslash j$, we can check that block $(i,j)$ is a dummy for $W'_{S, F}$, then by the 2d-dummy axiom, $\psi_{ij}(W'_{S, F})=0$. Especially, let $S=N\backslash i$ and any fixed $F'\subseteq M\backslash j$, we have:
\begin{equation*}
    \psi_{ij}(W_{N,F'})+\psi_{ij}(W_{N\backslash i,F'})=0.
\end{equation*}
For inductive purposes, assume it has been shown that $\psi_{ij}(S, F')+\psi_{ij}(S\cup i, F')=0$ for fixed $F'\subseteq M\backslash j$ and every $S\subseteq N\backslash i$ with $|S|\geq k\geq 2$. (The case $k=n-1$ has been proved.) Now take fixed $S\subseteq N\backslash i$ with $|S|=k-1$, then
\begin{align*}
    0=\psi_{ij}(W'_{S,F'})&=\sum_{S\subseteq S_1\subseteq N}\psi_{ij}(W_{S_1,F'})\\
    &=\psi_{ij}(W_{S\cup i,F'})+\psi_{ij}(W_{S ,F'})+\sum_{\substack{S_1\subseteq N\backslash i\\S\subsetneq S_1}}[\psi_{ij}(W_{S_1\cup i,F'})+\psi_{ij}(W_{S_1,F'})]\\
    &=\psi_{ij}(W_{S\cup i,F'})+\psi_{ij}(W_{S ,F'}).
\end{align*}
Therefore, $\psi_{ij}(W_{S\cup i,F'})+\psi_{ij}(W_{S ,F'})=0$ for all $S\subseteq N\backslash i$ and fixed $F'\subseteq N\backslash j$ with $0<|S|\leq n-1$ and $0<|F'|\leq m-1$.
Similarly, we  have another conclusion that $\psi_{ij}(W_{S',F})+\psi_{ij}(W_{S' ,F\cup j})=0$ for fixed $S'\subseteq N\backslash i$ and all $F\subseteq N\backslash j$ with $0<|S'|\leq n-1$ and $0<|F|\leq m-1$ by simply defining another similar utility function $W'_{S', F}$ and repeat the process above again.

Using the results above,
\begin{align*}
    \psi_{ij}(h)&=\sum_{\substack{S\subseteq N\\F\subseteq M}}h(S, F)\psi_{ij}(W_{S, F})\\
    &=\sum_{F\subseteq M}\sum_{S\subseteq N\backslash i}h(S\cup i,F)\psi_{ij}(W_{S\cup i, F})+h(S,F)\psi_{ij}(W_{S, F})\\
    &=\sum_{S\subseteq N\backslash i}\sum_{F\subseteq M}h(S\cup i,F)\psi_{ij}(W_{S\cup i, F})-h(S,F)\psi_{ij}(W_{S\cup i, F})\\
    &=\sum_{S\subseteq N\backslash i}\sum_{F\subseteq M\backslash j}\psi_{ij}(W_{S\cup i, F\cup j})[h(S\cup i,F\cup j)-h(S,F\cup j)]\\
    &-\sum_{S\subseteq N\backslash i}\sum_{F\subseteq M\backslash j}\psi_{ij}(W_{S\cup i, F\cup j})[h(S\cup i, F)-h(S, F)]\\
    &=\sum_{S\subseteq N\backslash i}\sum_{F\subseteq M\backslash j}\psi_{ij}(W_{S\cup i, F\cup j})[h(S\cup i,F\cup j)+h(S,F)\\
    &-h(S,F\cup j)-h(S\cup i,F)].
\end{align*}
For simplicity, denote $\psi_{ij}(W_{S\cup i, F\cup j})$ as $p^{ij}_{S,F}$, then
\begin{equation*}
    \psi_{ij}(h)=\sum_{S\subseteq N\backslash i}\sum_{F\subseteq M\backslash j}p^{ij}_{S,F}[h(S\cup i,F\cup j)+h(S,F)
    -h(S\cup i,F)-h(S,F\cup j)].
\end{equation*}
Consider the utility function $h_{ij}$, 
	\begin{equation*}
	    h_{ij}(W_1, W_2)=\left\{
	\begin{array}{lr}
		1, if \ i\in W_1, j\in W_2.\\
		0, otherwise.
	\end{array}
	\right.
	\end{equation*}
	and we can check that $ij$ is a dummy for $h_{ij}$, and $\psi_{ij}(h_{ij})=1$. Hence
\begin{equation*}
    1=\psi_{ij}(h_{ij})=\sum_{S\subseteq N\backslash i}\sum_{F\subseteq M\backslash j}p^{ij}_{S,F}.
\end{equation*}
\end{proof}
Next, add the 2d-symmetry axiom to Lemma \ref{lemma:linearity and dummy} and we make the conclusion that $p^{ij}_{S,F}$ is only related to the cardinality of $S$ and $F$, which is not associated with the name of the blocks.
\begin{lemma}\label{lemma: symmetry}
Assume 
Lemma \ref{lemma:linearity and dummy} holds.
If $\psi_{ij}$ also satisfies the 2d-symmetry axiom, then
\begin{equation*}
    p^{ij}_{S,F}=p_{s,f},
\end{equation*}
where $p_{s,f}$ is some common value for $S\subseteq N\backslash i$, $F\subseteq M\backslash j$ and $0\leq|S|=s\leq n-1$, $0\leq|F|=f\leq m-1$.
\end{lemma}
\begin{proof}
Define a utility $\hat{h}_{S,F}$:
	\begin{equation*}
	    \hat{h}_{S,F}(W_1, W_2)=\left\{
	\begin{array}{lr}
		1, if \ S\subsetneq W_1 , F\subsetneq W_2.\\
		0, otherwise.
	\end{array}
	\right.
	\end{equation*}
\begin{enumerate}
    \item For $i\in N$ and $j\in M$, let $S_1$, $F_1$ and $S_2$, $F_2$ be any two coalitions where $S_1, S_2\subseteq N\backslash i$ and $F_1, F_2\subseteq M\backslash j$ with $0<|S_1|=|S_2|< n-1$ and $0<|F_1|=|F_2|< m-1$ respectively. Consider two permutation $\pi_1$ and $\pi_2$ which satisfy $\pi_1(S_1)=S_2, \pi_1(i)=i$ and $\pi_2(F_1)=F_2, \pi_2(j)=j$. Then,
\begin{equation*}
    p_{S_1,F_1}^{ij}=\psi_{ij}(\hat{h}_{S_1,F_1})=\psi_{ij}(\hat{h}_{S_2,F_2})=p_{S_2,F_2}^{ij},
\end{equation*}
where the central equality is a consequence of the 2d-symmetry axiom.
\item For distinct $i_1, i_2\in N$ and $j_1, j_2\in M$, let $S\subseteq N\backslash \{i_1,i_2\}$ and $F\subseteq M\backslash \{j_1, j_2\}$, and the permutations $\pi_1, \pi_2$ respectively interchange $i_1, i_2$ and $j_1, j_2$ while leaving other elements fixed. Then,
\begin{equation*}
    \pi_1\pi_2\hat{h}_{S,F}=\hat{h}_{S,F},
\end{equation*}
\begin{equation*}
    p_{S,F}^{i_1j_1}=\psi_{i_1j_1}(\hat{h}_{S,F})=\psi_{i_2j_2}(\hat{h}_{S,F})=p_{S,F}^{i_2j_2},
\end{equation*}
where the central equality is a consequence of the 2d-symmetry axiom. Combining with the previous result in Step 1, we find that for every $0<s<n-1$ and $0<f<m-1$, there is a $p_{s,f}$ such that $p_{S,F}^{ij}=p_{s,f}$ for every $i\in N$ and $j\in M$, $S\subseteq N\backslash i$ and $F\subseteq M\backslash j$ with $|S|=s$, $|F|=f$.
\item Similarly, by using different utility functions, we can find for $\forall i\in N, j\in M$:
\begin{itemize}
    \item a $p_{n-1,f}$ such that $p_{N\backslash i,F}^{ij}=p_{n-1,f}$ for $F\subseteq M\backslash j$ and $0\leq|F|=f<m-1$,
    \item a $p_{s,m-1}$ such that $p_{S,M\backslash j}^{ij}=p_{s,m-1}$ for $S\subseteq N\backslash i$ and $0\leq|S|=s<n-1$,
    \item a $p_{0,f}$ such that $p_{\emptyset,F}^{ij}=p_{0,f}$ for $F\subseteq M\backslash j$ and $0<|F|=f<m-1$,
    \item a $p_{s,0}$ such that $p_{S, \emptyset}^{ij}=p_{s,0}$, for $S\subseteq N\backslash i$ and $0<|S|=s<n-1$,
    \item a $p_{n-1,m-1}$ such that $p_{N\backslash i, M\backslash j}^{ij}=p_{n-1,m-1}$,
    \item a $p_{0,0}$ such that  $p_{\emptyset,\emptyset}^{ij}=p_{0,0}$ which makes the sum of all the weights equals to 1.
\end{itemize}
\end{enumerate}
\end{proof}
Finally add the 2d-efficiency axiom and obtain the uniqueness of $\alg$.
\begin{lemma}\label{lemma: efficiency}
Assume 
Lemma \ref{lemma:linearity and dummy} holds.
Then $\psi_{ij}(h)$ satisfies the 2d-efficiency axiom if and only if 
\begin{equation}\label{efficiency 1}
    \sum_{\substack{i\in N\\j\in M}}p_{N\backslash i,M\backslash j}^{ij}=1,
\end{equation}
\begin{equation}\label{efficiency 2}
    \sum_{\substack{i\in S\\j\in F}}p_{S\backslash i,F\backslash j}^{ij}+\sum_{\substack{i\notin S\\j\notin F}}p_{S,F}^{ij}
    -\sum_{\substack{i\notin S\\j\in F}}p_{S,F\backslash j}^{ij}
    -\sum_{\substack{i\in S\\j\notin F}}p_{S\backslash i,F}^{ij}=0,
\end{equation}
where $S\subsetneq N$ or $F\subsetneq M$.
\end{lemma}
\begin{proof}
On the one hand, by Eq.~(\ref{efficiency 1}) and Eq.~(\ref{efficiency 2}),
\begin{align*}
    h(N,M)&=\sum_{\substack{S\subseteq N\\F\subseteq M}}h(S,F)[\sum_{\substack{i\in S\\j\in F}}p_{S\backslash i,F\backslash j}^{ij}+\sum_{\substack{i\notin S\\j\notin F}}p_{S,F}^{ij}-\sum_{\substack{i\notin S\\j\in F}}p_{S,F\backslash j}^{ij}-\sum_{\substack{i\in S\\j\notin F}}p_{S\backslash i,F}^{ij}]\\
    &=\sum_{\substack{i\in N\\j\in M}}\sum_{\substack{S\subseteq N\backslash i\\F\subseteq M\backslash j}}p_{S,F}^{ij}[h(S\cup i,F\cup j)+h(S,F)
    -h(S\cup i,F)-h(S,F\cup j)]\\
    &=\sum_{\substack{i\in N\\j\in M}}\psi_{ij}(h).
\end{align*}
On the other hand, recall:
\begin{equation*}
	    \hat{h}_{S,F}(W_1, W_2)=\left\{
	\begin{array}{lr}
		1, if \ S\subsetneq W_1 , F\subsetneq W_2.\\
		0, otherwise.
	\end{array}
	\right.
	\end{equation*}
	and
	\begin{equation*}
	    h_{S,F}(W_1, W_2)=\left\{
	\begin{array}{lr}
		1, if \ S\subseteq W_1 , F\subseteq W_2.\\
		0, otherwise.
	\end{array}
	\right.
	\end{equation*}
Consider two new utility functions
\begin{equation*}
	\tilde{h}_{S,F}(W_1, W_2)=\left\{
	\begin{array}{lr}
		1, if \ S\subsetneq W_1 , F\subseteq W_2,\\
		0, otherwise.
	\end{array}
	\right.
	\end{equation*}
and
\begin{equation*}
	\bar{h}_{S,F}(W_1, W_2)=\left\{
	\begin{array}{lr}
		1, if \ S\subseteq W_1 , F\subsetneq W_2,\\
		0, otherwise.
	\end{array}
	\right.
	\end{equation*}
Then for any $S\subseteq N$, $F\subseteq M$,
\begin{align*}
    &\sum_{\substack{i\in N\\j\in M}}\psi_{ij}(h_{S,F})+\sum_{\substack{i\in N\\j\in M}}\psi_{ij}(\hat{h}_{S,F})-\sum_{\substack{i\in N\\j\in M}}\psi_{ij}(\tilde{h}_{S,F})-\sum_{\substack{i\in N\\j\in M}}\psi_{ij}(\bar{h}_{S,F})\\
    &=\sum_{\substack{i\in S\\j\in F}}p_{S\backslash i,F\backslash j}^{ij}+\sum_{\substack{i\notin S\\j\notin F}}p_{S,F}^{ij}
    -\sum_{\substack{i\notin S\\j\in F}}p_{S,F\backslash j}^{ij}
    -\sum_{\substack{i\in S\\j\notin F}}p_{S\backslash i,F}^{ij}.
\end{align*}
When $S=N$ and $F=M$, 
\begin{align*}
    &\sum_{\substack{i\in N\\j\in M}}\psi_{ij}(h_{N,M})+\sum_{\substack{i\in N\\j\in M}}\psi_{ij}(\hat{h}_{N,M})-\sum_{\substack{i\in N\\j\in M}}\psi_{ij}(\tilde{h}_{N,M})-\sum_{\substack{i\in N\\j\in M}}\psi_{ij}(\bar{h}_{N,M})\\
    &=h_{N,M}(N,M)+\hat{h}_{N,M}(N,M)-\tilde{h}_{N,M}(N,M)-\bar{h}_{N,M}(N,M)\\
    &=1,
\end{align*}
Otherwise,
\begin{align*}
    &\sum_{\substack{i\in N\\j\in M}}\psi_{ij}(h_{S,F})+\sum_{\substack{i\in N\\j\in M}}\psi_{ij}(\hat{h}_{S,F})-\sum_{\substack{i\in N\\j\in M}}\psi_{ij}(\tilde{h}_{S,F})-\sum_{\substack{i\in N\\j\in M}}\psi_{ij}(\bar{h}_{S,F})\\
    &=h_{S,F}(N,M)+\hat{h}_{S,F}(N,M)-\tilde{h}_{S,F}(N,M)-\bar{h}_{S,F}(N,M)\\
    &=0.
\end{align*}
Hence, Eq.~(\ref{efficiency 1}) and Eq.~(\ref{efficiency 2}) can be easily obtained.
\end{proof}

Now, let's  prove Theorem \ref{thm: Representation of 2D Shapley-value}.

\begin{proof}[Proof of Theorem \ref{thm: Representation of 2D Shapley-value}]
By Lemma \ref{lemma: symmetry},
\begin{align*}
     \psi_{ij}(h)&=\sum_{s=0}^{n-1}\sum_{f=0}^{m-1}\sum_{\substack{S\subseteq N\backslash i\\|S|=s}}\sum_{\substack{F\subseteq M\backslash j\\|F|=f}}p_{s,f}[h(S\cup i, F\cup j)+h(S, F)\\
	 &-h(S\cup i, F)-h(S, F\cup j)].
\end{align*}
By Lemma \ref{lemma:linearity and dummy} and Lemma \ref{lemma: efficiency}, we have the following equations:
\begin{equation}\label{linear recursive condition}
    \begin{aligned}
    &\sum_{s=0}^{n-1}\sum_{f=0}^{m-1}\tbinom{n-1}{s}\tbinom{m-1}{f}p_{s,f}=1,\\
    &sf\cdot p_{s-1,f-1}+(n-s)(m-f)\cdot p_{s,f}=(n-s)f\cdot p_{s,f-1}\\
    &+s(m-f)p_{s-1,f}, \ 1\leq s\leq n-1, 1\leq f\leq m-1,\\
    &(m-f)\cdot p_{0,f}=f\cdot p_{0,f-1}, \ 1\leq f\leq m-1,\\
    &(n-s)\cdot p_{s,0}=s\cdot p_{s-1,0}, \ 1\leq s\leq n-1,\\
    &nm\cdot p_{n-1,m-1}=1.
\end{aligned}
\end{equation}
Actually, we can omit the first equation and the conditions are:
\begin{equation}\label{linear recursive condition1}
    \begin{aligned}
    &sf\cdot p_{s-1,f-1}+(n-s)(m-f)\cdot p_{s,f}=(n-s)f\cdot p_{s,f-1}\\
    &+s(m-f)p_{s-1,f}, \ 1\leq s\leq n-1, 1\leq f\leq m-1,\\
    &(m-f)\cdot p_{0,f}=f\cdot p_{0,f-1}, \ 1\leq f\leq m-1,\\
    &(n-s)\cdot p_{s,0}=s\cdot p_{s-1,0}, \ 1\leq s\leq n-1,\\
    &nm\cdot p_{n-1,m-1}=1.
\end{aligned}
\end{equation}
Hence, we have $n\cdot m$ variables and $(m-1)(n-1)+(m-1)+(n-1)+1=n\cdot m$ equations.

Eq.~(\ref{linear recursive condition1}) has a solution:
\begin{equation}\label{solution}
    p_{s,f}=\frac{s!(n-s-1)!}{n!}\cdot\frac{f!(m-f-1)!}{m!}.
\end{equation}
Therefore,
\begin{align*}
    \psi_{ij}(h)&=\sum_{s=0}^{n-1}\sum_{f=0}^{m-1}\sum_{\substack{S\subseteq N\backslash i\\|S|=s}}\sum_{\substack{F\subseteq M\backslash j\\|F|=f}}\frac{s!(n-s-1)!}{n!}\cdot\frac{f!(m-f-1)!}{m!}[h(S\cup i, F\cup j)+h(S, F)\\
	 &-h(S\cup i, F)-h(S, F\cup j)]\\
	 &=\frac{1}{nm}\sum_{s=1}^{n}\sum_{f=1}^{m}\sum_{\substack{S\subseteq N\backslash i\\|S|=s-1}}\sum_{\substack{F\subseteq M\backslash j\\|F|=f-1}}\frac{(s-1)!(n-s)!}{(n-1)!}\cdot\frac{(f-1)!(m-f)!}{(m-1)!}[h(S\cup i, F\cup j)+h(S, F)\\
	 &-h(S\cup i, F)-h(S, F\cup j)]\\
	 &=\frac{1}{nm}\sum_{s=1}^{n}\sum_{f=1}^{m}\frac{1}{\tbinom{n-1}{s-1}\tbinom{m-1}{f-1}}\sum_{(S, F)\in D_{sf}^{ij}}[h(S\cup i, F\cup j)+h(S, F)-h(S\cup i, F)-h(S, F\cup j)]\\
	 &=\frac{1}{nm}\sum_{s=1}^n\sum_{f=1}^m\Delta_{sf}.
\end{align*}
\begin{align*}
    \psi_{ij}(h)=\sum_{s=0}^{n-1}\sum_{f=0}^{m-1}\sum_{\substack{S\subseteq \mathcal{N}\backslash i\\|S|=s}}\sum_{\substack{F\subseteq \mathcal{M}\backslash j\\|F|=f}}&p_{s,f}[h(S\cup i, F\cup j)+h(S, F)\\
    &-h(S\cup i, F)-h(S, F\cup j)]\\   &=\sum_{s=0}^{n-1}\sum_{f=0}^{m-1}\tbinom{n-1}{s}\tbinom{m-1}{f}p_{s,f}\frac{1}{\tbinom{n-1}{s}\tbinom{m-1}{f}}\sum_{(S, F)\in D_{sf}^{ij}}[h(S\cup i, F\cup j)+h(S, F)-h(S\cup i, F)-h(S, F\cup j)]\\
\end{align*}
Now we prove the solution Eq.~(\ref{solution}) is unique.

Convert the Eq.~(\ref{linear recursive condition1}) to matrix equations in the form of $$\mA \rvx=\rvb,$$where $$\rvx^T=(p_{0,0},p_{0,1},\dots,p_{0,m-1},p_{1,0},p_{1,1},\dots,p_{1,m-1},\dots,p_{n-1,0},\dots,p_{n-1,m-1})_{1\times nm},$$ $$\rvb^T=(0,0,0,\dots,0,1)_{1\times nm},$$ and
\begin{equation}
    \mA=
\begin{pmatrix}
\mA_1\\
\mA_2
\end{pmatrix}_{nm\times nm},
\end{equation}
where
\begin{align*}
&\mA_1=
    \begin{pmatrix}
    \mA^0_{(m-1)\times m}&\mO_{(m-1)\times m}&\cdots&\cdots&\mO_{(m-1)\times m}\\
    \\
    \mA^1_{m\times m}&\mB^1_{m\times m}&\mO_{m\times m}&\cdots&\mO_{m\times m}\\
    \\
    \mO_{m\times m}&\mA^2_{m\times m}&\mB^2_{m\times m}&\cdots&\mO_{m\times m}\\
    \\
    \vdots                &\vdots         &\ddots         & \ddots &\vdots\\
    \\
    \mO_{m\times m}&\mO_{m\times m}&\cdots&\mA^{n-1}_{m\times m}&\mB^{n-1}_{m\times m}
    \end{pmatrix}_{(nm-1)\times nm },\\
    \\
    &\mA_2=\Big(0,0,\cdots,0,nm\Big)_{1\times nm}.
\end{align*}
And
\begin{tiny}
\begin{align*}
\mA^0_{(m-1)\times m}&=
    \begin{pmatrix}
    1&-(m-1)&0&\cdots&\cdots&0\\
    \\
    0&2&-(m-2)&0&\cdots&0\\
    \\
    0&0&3&-(m-3)&\cdots&0
    \\
    \vdots& \vdots&\vdots&\ddots&\ddots&\vdots\\
    \\
    0& 0& 0&\cdots&m-1&-1
    \end{pmatrix}_{(m-1)\times m},\\
    ~\\
    \mA^j_{m\times m}&=
    \begin{pmatrix}
    j&0&0&\cdots&\cdots&0\\
    \\
    j&-j\cdot(m-1)&0&0&\cdots&0\\
    \\
    0&2j&-j\cdot(m-2)&0&\cdots&0\\
    \\
    0&0&3j&-j\cdot(m-3)&\cdots&0\\
    \\
    \vdots& \vdots&\vdots&\ddots&\ddots&\vdots\\
    \\
    0& 0& 0&\cdots&j\cdot(m-1)&-j
    \end{pmatrix}_{m\times m},\ 1\leq j\leq n-1,\\
    ~\\
    \mB^j_{m\times m}&=
    \begin{pmatrix}
    -(n-j)&0&0&\cdots&\cdots&0\\
    \\
    -(n-j)&(n-j)\cdot(m-1)&0&0&\cdots&0\\
    \\
    0&-2\cdot (n-j)&(n-j)\cdot(m-2)&0&\cdots&0\\
    \\
    0&0&-3\cdot(n-j)&(n-j)\cdot(m-3)&\cdots&0\\
    \\
    \vdots& \vdots&\vdots&\ddots&\ddots&\vdots\\
    \\
    0& 0& 0&\cdots&-(m-1)\cdot(n-j)&n-j
    \end{pmatrix}_{m\times m},\ 1\leq j\leq n-1.
\end{align*}
\end{tiny}
For example, if $n=m=3$, then 
\begin{equation*}
    \mA=
    \begin{pmatrix}
    \begin{array}{ccc:ccc:ccc}
     1&-2&0&0&0&0&0&0&0\\
    0&2&-1&0&0&0&0&0&0\\ \hline
    1&0&0&-2&0&0&0&0&0\\ 
    1&-2&0&-2&4&0&0&0&0\\
    0&2&-1&0&-4&2&0&0&0\\ \hline
    0&0&0&2&0&0&-1&0&0\\
    0&0&0&2&-4&0&-1&2&0\\
    0&0&0&0&4&-2&0&-2&1\\ \hline
    0&0&0&0&0&0&0&0&9
    \end{array}
    \end{pmatrix}_{9\times9}
\end{equation*}
Convert $\mA$ to $\hat{\mA}$ by using the elementary column and row transformation,
\begin{equation*}
    \hat{\mA}=
    \begin{pmatrix}
    \begin{array}{ccc:ccc:ccc}
    1&-2&0&2&-4&0&1&-2&0\\
    0&2&-1&0&4&-2&0&2&-1\\\hline
    1&0&0&0&0&0&0&0&0\\
    1&-2&0&0&0&0&0&0&0\\
    0&2&-1&0&0&0&0&0&0\\\hline
    0&0&0&1&0&0&0&0&0\\
    0&0&0&1&-2&0&0&0&0\\
    0&0&0&0&2&-1&0&0&0\\\hline
    0&0&0&0&0&0&0&0&1
    \end{array}
    \end{pmatrix}_{9\times9}.
\end{equation*}
According to the property of the elementary row and column transformation, $$Rank(\mA)=Rank(\hat{\mA}).$$ 
Consider equation $$\hat{\mA}\rvx=\mathbf{0},$$
and the solution is only $\rvx=\mathbf{0}$, hence $$Rank(\mA)=Rank(\hat{\mA})=9.$$ 

In general, we can prove $Rank(\mA)=nm$ always holds for any $n\geq 1$ and $m\geq 1$, (Make elementary column transformation for $[\mA^j_{m\times m}, \mB^j_{m\times m}]$ in the context of $\mA$ with the order of $j=1,2,\dots, n-1$.) Hence the solution of Eq.~(\ref{linear recursive condition1}) is unique, which is shown in Eq.~(\ref{solution}). And we can check Eq.~(\ref{solution}) also satisfies Eq.~(\ref{linear recursive condition}), hence the solution of Eq.~(\ref{linear recursive condition}) is unique.
\end{proof}

\section{Proof of Corollary~\ref{cor: 2d to 1d}}\label{appendix: cor}
\begin{proof}
We use the same technique in the proof of Lemma~\ref{lemma: efficiency}.
\begin{align*}
    \psi_{i\cdot}^{1d}(h)&=\sum_{j\in M}\sum_{\substack{S\subseteq N\backslash i\\F\subseteq M\backslash j}}p_{s,f}[h(S\cup i, F\cup j)+h(S,F)-h(S\cup i,F)-h(S,F\cup j)]\\
    &=\sum_{\substack{S\subseteq N\backslash i\\F\subseteq M}}h(S\cup i,F)[\sum_{j\in F}p_{s,f-1}-\sum_{j\notin F}p_{s,f}]+h(S,F)[\sum_{j\notin F}p_{s,f}-\sum_{j\in F}p_{s,f-1}]\\
    &=\sum_{\substack{S\subseteq N\backslash i\\F\subseteq M}}(\sum_{j\in F}p_{s,f-1}-\sum_{j\notin F}p_{s,f})[h(S\cup i,F)-h(S,F)]\\
    &=\sum_{S\subseteq N\backslash i}(\sum_{j\in M}p_{s,m-1})[h(S\cup i,M)-h(S,M)].
\end{align*}
Substitute Eq.~(\ref{solution}) into the above equation and we get the conclusion. The similar argument can be applied to $\psi_{\cdot j}^{1d}$.
\end{proof}

\section{Permutation-based $\alg$ Formulation}
To compute $\alg$ more efficient, we propose the following corollary.
\begin{corollary} \label{cor:eff_2d}
Eq.~(\ref{eqn: 2D Shapley}) has an equivalent form as follows:
\begin{align}\label{Eq:2dShap_set_express}
    &\psi_{ij}^{2d}=\frac{1}{nm}\sum_{\substack{S\subseteq N\backslash i\\F\subseteq M\backslash j}} \frac{[h(S\cup i, F\cup j)+h(S, F)-h(S\cup i, F)-h(S, F\cup j)]}{\tbinom{n-1}{|S|}\tbinom{m-1}{|F|}},
\end{align}
or
\begin{align} \label{eq:2dperm}
    &\psi_{ij}^{2d}=\frac{1}{n!m!}\sum_{\substack{\pi_1 \in \Pi(N) \\ \pi_2 \in \Pi(M) }} 
    [h(P_{i}^{\pi_1} \cup i, P_{j}^{\pi_2} \cup j) + h(P_{i}^{\pi_1}, P_{j}^{\pi_2}) - h(P_{i}^{\pi_1}\cup i, P_{j}^{\pi_2} )- h(P_{i}^{\pi_1}, P_{j}^{\pi_2}\cup j )],
\end{align}
where $\Pi(A)$ denotes a set of all permutations of $A$ and $P_{k}^{\pi}$ a set of all elements of $A$ that precede $k \in A$ in the permutation $\pi \in \Pi(A)$. 
\end{corollary}

The formulation in Eq.~(\ref{Eq:2dShap_set_express}) is a simple derivation from Eq.~(\ref{eqn: 2D Shapley}) that sums marginal contributions over all subsets. Whereas, the second formulation in Eq.~(\ref{eq:2dperm}) sums over all sample and feature permutations, and the marginal contribution of block $(i,j)$ is weighted by a coefficient that measures all orderings of samples appearing before and after sample $i$ and all orderings of features appearing before and after feature $j$.
This corollary gives a simple expression of $\alg$. Using this equivalent formulation, we can design efficient algorithms for $\alg$ implementation.

\section{Algorithm Details} \label{sec:algo}

Here, we explain the implementation of algorithms and explore ways to achieve efficient computation.

\subsection{Saving Computation in $\algmc$}
 First, we focus on $\algmc$. 
Apart from Monte Carlo sampling on both sample and feature permutations to reduce complexity, we also reduce the number of model training to a single time for each counterfactual evaluation as opposed to $4$, which is derived in Eq.~\ref{eq:margin}. 
Let us observe that in the marginal contribution equation, we have $4$ utility terms, but actually, $3$ of them are already computed, which we can reuse them. 
We take a pair $(i,j)$ as an example. 
For the marginal contribution of $(i,j)$, we have 4 utility terms to compute: $h(S\cup i, F\cup j), h(S, F\cup j), h(S\cup i, F), h(S, F)$. However, we notice that $h(S, F\cup j)$ was already computed for a pair $(i-1,j)$, $h(S\cup i, F)$ for a pair $(i,j-1)$, and $h(S, F)$ for $(i-1,j-1)$. Therefore, by saving these evaluations, we can reduce the total number of model training by $75\%$. 
Saving all model evaluations for every block might overflow the memory. However, we only need to save the utilities of the previous and current rows (columns) if we are looping horizontally downwards (vertically rightwards), which promotes efficient memory usage. Additionally, our algorithm can be parallelized. In particular, every permutation can be computed independently and combined at the last stage, which is the ``while loop'' in Algorithm~\ref{alg:2d-mc}.

\subsection{Limitations of $\algmc$ and Possible Improvements}
\edit{
One limitation of the Monte Carlo method is time complexity which scales with the number of rows and columns in an aggregate data matrix. To improve the efficiency of $\algmc$, we can reduce the burden on model retraining of $\algmc$ to lower the computation cost. For example, there exist highly efficient methods for model re-training, such as FFCV [1,2], which has been applied in Datamodels [3] and can significantly reduce computation complexity. Another limitation is that $\algmc$  relies on the performance scores associated with models trained on different subsets to determine the cell values. However, these values are susceptible to noise due to training stochasticity when the learning algorithm is randomized (e.g., SGD) \citep{wang2022data}.
To overcome these limitations, we proposed an efficient, nearest-neighbor-based method, $\algknn$, which involves no model training and only requires sorting data. With this method, we also avoid the problem of model training stochasticity, which $\algmc$ is facing with. Another advantage of $\algknn$ is that it has an explicit formulation for sample values and only requires permuting over features. This method not only beats $\algmc$ by an order of magnitude in terms of computational efficiency but is straightforward to compute and only requires CPU resources.}



\subsection{Saving Computation in $\algknn$}

Apart from removing the dependency on the sample permutations and all model training, $\algknn$ can further be reduced in computation. Similar to the $\algmc$, we here also save the utility terms, as shown in Algorithm~\ref{alg:2d-knn}. For each pair $(i,j)$, we need to compute $SV_{KNN}(i, P_j^{\pi} \cup k)$ and  $SV_{KNN}(i, P^{\pi}_j)$. However, the second term was already calculated for the previous feature in $\pi$ prior to $j$. Thus, we can reduce the total number of $SV_{KNN}$ evaluations by $50\%$.

\begin{algorithm}[ht]
\SetAlgoLined
\SetKwInput{Input}{Input}
\SetKwInput{Output}{Output}
\SetKwInput{Ensure}{Ensure}
\Input{ Training Set $D$, Learning Algorithm $\mathcal{A}$, Test Set $T$, Utility Function $h$.}
\Output{ Sample-Feature 2D Shapley Values $\psi^{2d}$.}
\Ensure{$\forall_{i,j}$, $\psi_{ij}^{2d}=0$; $t=0$.} 
\While {$\psi^{2d}$ not converged} {
    \indent $\pi_N \leftarrow $ Random Samples Permutation  
    
    $\pi_M \leftarrow $ Random Features Permutation
    
    $u \leftarrow \textbf{0}$ \small{\tcp{Utility Matrix}}
    \For {$i,j$ in range($\pi_N), range(\pi_M$)} {
         $s \leftarrow \pi_N(i), f \leftarrow \pi_M(j)$\\
         
         $u[s,f] \leftarrow h\left(P^{\pi_N}_s \cup \{s\},P^{\pi_M}_f \cup \{f\}\right)$\\
         
         $\psi_{sf}^{new} \leftarrow u[s,f] + u[\pi_N(i-1),\pi_M(j-1)] - u[\pi_N(i),\pi_M(j-1)] - u[\pi_N(i-1),\pi_M(j)]$\\
         
         $\psi_{sf}^{2d} \leftarrow \frac{t}{t+1} \psi_{sf}^{2d} + \frac{1}{t+1} \psi_{sf}^{new} $
    }
    Set $t \leftarrow t+1$
}
\caption{ $\algmc$ Valuation Algorithm.}
\label{alg:2d-mc}
\end{algorithm}


\begin{algorithm}[ht]
\SetAlgoLined
\SetKwInput{Input}{Input}
\SetKwInput{Output}{Output}
\SetKwInput{Ensure}{Ensure}
\Input{ Training Set $D$, Test Set $T$, Top $K$.}
\Output{ Sample-Feature 2D Shapley Values $\psi^{2d}$.}
\Ensure{$\forall_{i,j}$, $\psi_{ij}^{2d}=0$; $t=0$.} 
\While {$\psi^{2d}$ not converged} { 
    
    $\pi_M \leftarrow $ Random Features Permutation 
    
    $u \leftarrow \textbf{0}$ \small{\tcp{$SV_{knn}$ values}}
    \For {$j$ in range($\pi_M$)} {
         $f \leftarrow \pi_M(j)$\\
         
         $u[f] \leftarrow SV_{KNN}(N,P_m^{\pi_M} \cup \{f\}, T)$\\
         
         $\psi_{sf}^{new} \leftarrow u[f]_s - u[\pi_M(j-1)]_s$ \\
         
         $\psi_{sf}^{2d} \leftarrow \frac{t}{t+1} \psi_{sf}^{2d} + \frac{1}{t+1} \psi_{sf}^{new} $
    }
    Set $t \leftarrow t+1$
}
\caption{ $\algknn$ Valuation Algorithm.}
\label{alg:2d-knn}
\end{algorithm}

\subsection{Actual Runtime Complexity}

\edit{Time complexity is an important aspect when evaluating the efficiency of algorithms. In our case, we focus on determining the runtime of our methods for different number of cell valuations on the Census dataset until the values' convergence is achieved. While computing the runtime for the exact 2D Shapley runtime, we encounter a challenge due to the exponential growth of permutations with the cell size, making exact 2D Shapley intractable to compute. 
To address this, we benchmark the exact 2D Shapley runtime, by measuring the runtime for a single permutation and scale it by the total number of permutations needed for the exact 2D Shapley. As we observe in Table~\ref{fig:runtime}, $\algknn$, exhibits exceptional efficiency compared to $\algmc$ 
across various cell valuations on the Census dataset. At 1,000 cells valuation, $\algknn$ was at least 25 times faster than $\algmc$, showcasing a substantial advantage. Furthermore, as the number of cells increased to 100,000, $\algknn$ demonstrates a remarkable speed advantage, being approximately 300 times faster than $\algmc$. These findings clearly establish an advantage of $\algknn$ over $\algmc$ in terms of runtime efficiency. Moreover, we observe that both $\algknn$ and $\algmc$ outperform the exact 2D Shapley method in terms of runtime.
These results highlight the effectiveness and practicality of our approach for computing 2D-Shapley in real-world cases.}

\begin{table}[H]
\centering{
\begin{tabular}{l|c|c|c|c|c|c}

\textbf{Method}   & \textbf{1K} & \textbf{5K} & \textbf{10K} & \textbf{20K} & \textbf{50K} & \textbf{100K} \\ \midrule
\begin{tabular}{@{}l@{}}2D Shapley-Exact \\ (Theoretical)\end{tabular}   & 1.5E+301s         & 2.0E+1505s &	2.8E+3010s 	&5.6E+6020s 	& 4.4E+15051s 	&1.4E+30103s       \\
$\algmc$     & 280s 	& 1,661s 	& 3,127s 	& 9,258s 	 & 17,786s 	& 26,209s      \\ 
$\algknn$ & 11s 	& 25s 	& 37s 	 & 44s & 	53s 	& 88s       \\ 
\end{tabular}
}
\caption{Actual runtime comparison between 2D-Shapley methods.}
\label{fig:runtime}
\end{table}



\section{Implementation Details \& Results} \label{sec:res}


\subsection{Details on Datasets and Models}

For our experiments, we use the following datasets from Machine Learning Repository~\cite{Dua:2019}:

\begin{table}[H]
\centering{
\begin{tabular}{l|c|c|c}

\textbf{Dataset}   & \textbf{Training Data} & \textbf{Test Data} & \textbf{Features} \\ \midrule
Census Income                      & 32561         & 16281     & 14       \\
Default of Credit Card Clients     & 18000         & 12000     & 24       \\ 
Heart Failure & 512           & 513       & 13       \\ 
Breast Cancer Wisconsin (Original) & 242           & 241       & 10       \\ 
Wine Dataset & 106           & 72       & 13       \\ 
\end{tabular}
}
\caption{Details on datasets used in experiments.}
\label{fig:datasets}
\end{table}

In Breast Cancer Wisconsin dataset, we removed ``ID number'' from the list of features as it was irrelevant for model training.

For methods requiring model training, $\oned$, \texttt{Random}, and $\algmc$, we implemented a decision tree classifier on all of them.

Empirically, we verified that for each of the method, the cell values converge within $500$ permutations and that is the number we decide to use to run these methods.

Due to varying sizes of each dataset with different number of features, we set a different number of cells to be removed at a time. For bigger datasets, Census Income and Credit Default, we remove ten cells at a time, and for a smaller dataset, Breast Cancer, we remove one cell at a time.

\subsection{Additional Results on \emph{Sanity check of cell-wise values} experiment} \label{sec:sanity}

\begin{figure}[h]
\begin{center}
  \includegraphics[width=450pt]{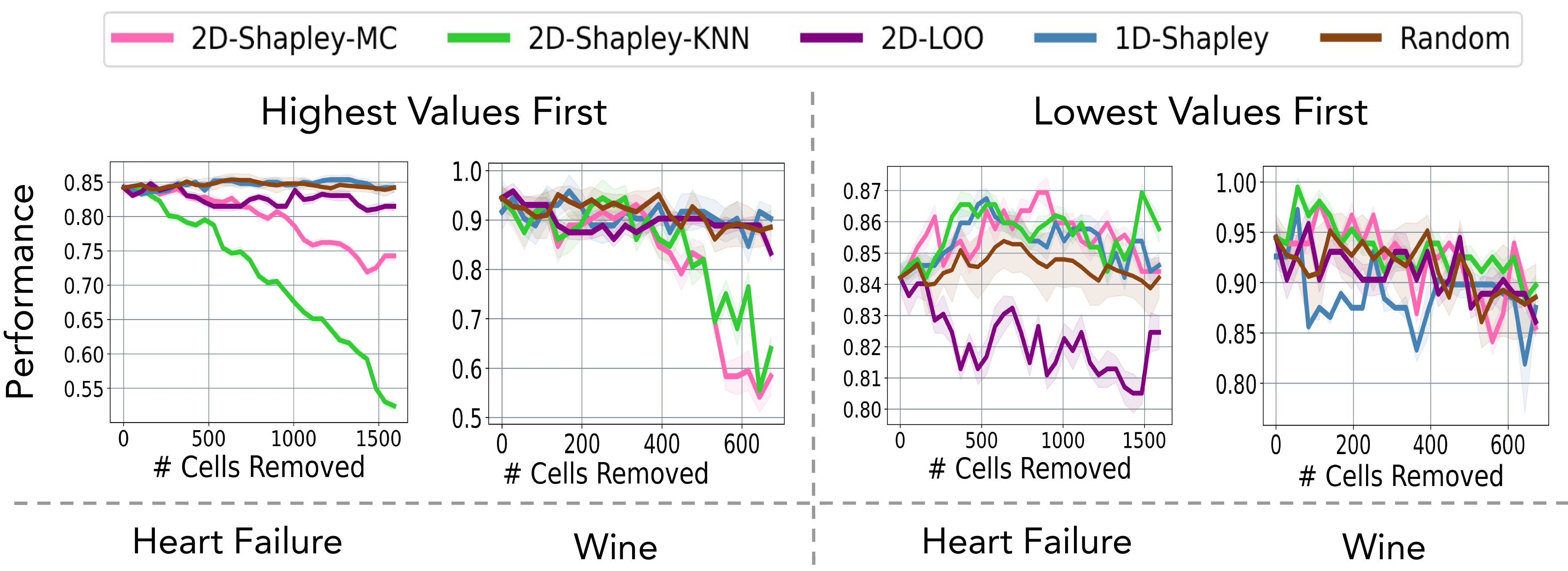}
  \caption{$\alg$ values for benign patients in the original breast cancer dataset. The green border denotes a cell before an outlier value has been injected to that cell.}~\label{fig:heart_wine}
  \end{center}
\end{figure}

\edit{We provide results on additional datasets, Heart Failure and Wine Dataset, to demonstrate the effectiveness of $\alg$ in cell-wise valuation. We additionally include the 2D LOO baseline for comparison. As we can observe in Figure~\ref{fig:heart_wine}, 2D LOO performance is comparable to or worse than the Random baseline. One of the main reasons is that 2D LOO only valuates a cell's contribution when all other cells are present. This means that after the sequential removal of some cells, the values obtained from 2D LOO may no longer accurately represent the importance of the cells. In contrast, our method computes a cell's value by averaging its contribution over various sample and feature subset sizes, which ensures our cell values are informative even after the sequential removal of a certain amount of cells, thereby addressing the shortcomings of 2D LOO and leading to improved performance in cell-wise valuation.
}

\subsection{Additional Details and Results on \emph{Fine-Grained Outlier Localization} experiment} \label{sec:outlier}


\subsubsection{Outlier Value Generation}

Our outlier generation technique is inspired by~\cite{du2022vos}. Specifically, for a random cell with a sample index $i$ and a feature index $j$, we generate an outlier value based on its feature $j$. We first recreate a distribution of the feature $j$ and then sample a value from a low-probability-density region, below $5\%$ in our experiment. 

\subsubsection{Heatmaps Comparison}

To better understand the detection rate of outlier values, we visualize them through a heatmap. In Figure~\ref{fig:bc_clean_heatmap}, we provide a $\alg$ heatmap of the original dataset before outlier injection and compare with a $\alg$ heatmap in Figure~\ref{fig:bc_2d_out_heatmap} after injecting outliers. Due to dimensional reasons, we transpose the heatmap, where the rows represent features and the columns denote the samples.

We observe through the breast cancer dataset that the cells with injected outliers have changed their values and lie mostly in the lower range of $\alg$ values. However, we can also notice that other cells are also affected by the outliers and the overall range of values has increased in both directions.

\begin{figure}[htb]
\begin{center}
  \includegraphics[width=400pt]{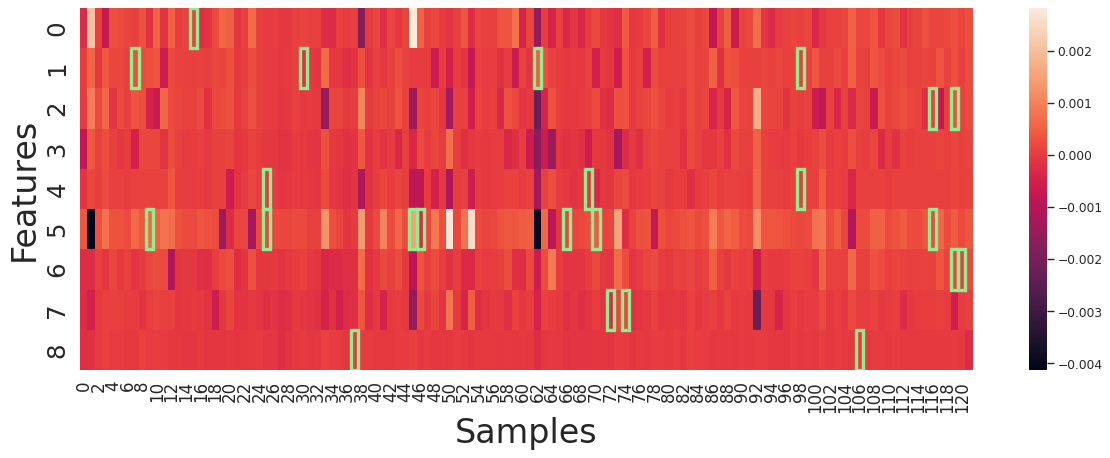}
  \caption{$\alg$ values for benign patients in the original breast cancer dataset. The green border denotes a cell before an outlier value has been injected to that cell.}~\label{fig:bc_clean_heatmap}
  \end{center}
\end{figure}

\begin{figure}[htb]
\begin{center}
  \includegraphics[width=400pt]{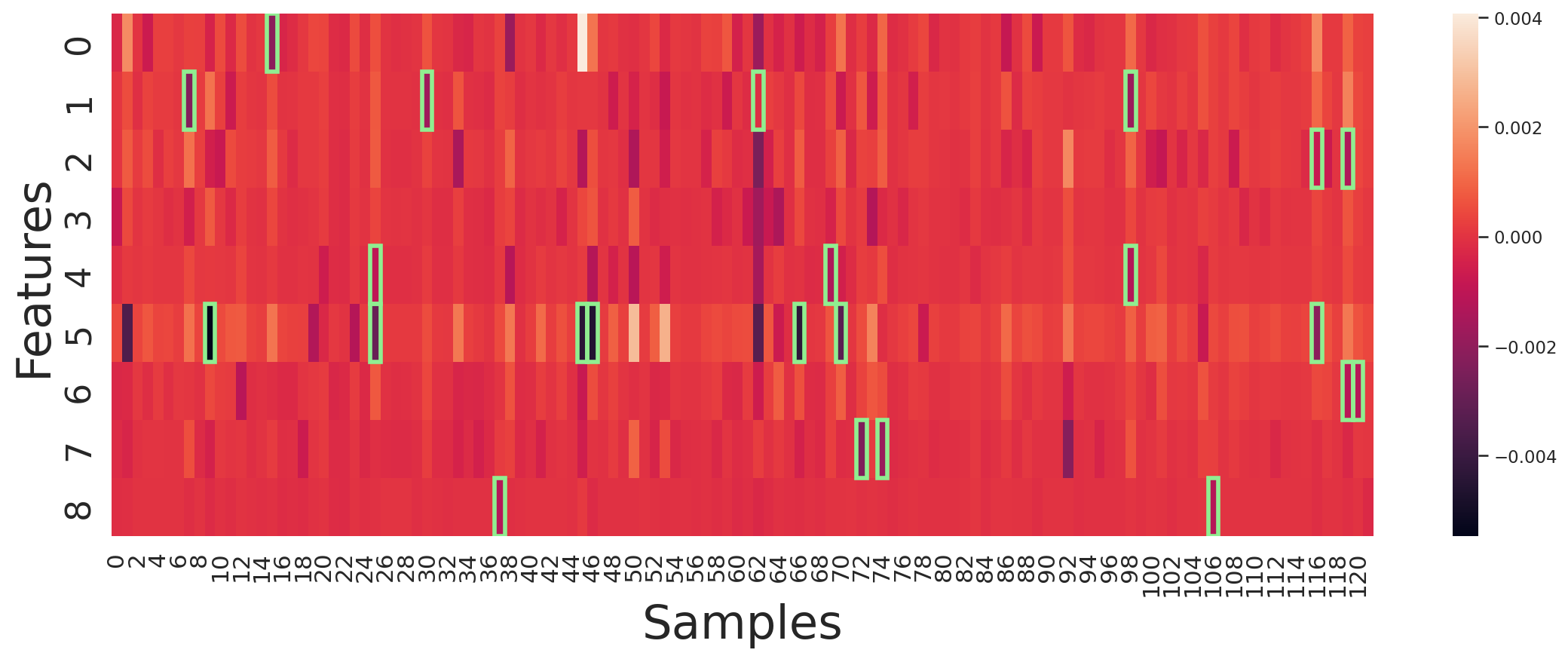}
  \caption{$\alg$ values for benign patients in the breast cancer dataset with randomly inserted outliers. The green border denotes a cell after an outlier value has been injected to that cell.}~\label{fig:bc_2d_out_heatmap}
  \end{center}
\end{figure}

\begin{figure}[htb]
\begin{center}
  \includegraphics[width=400pt]{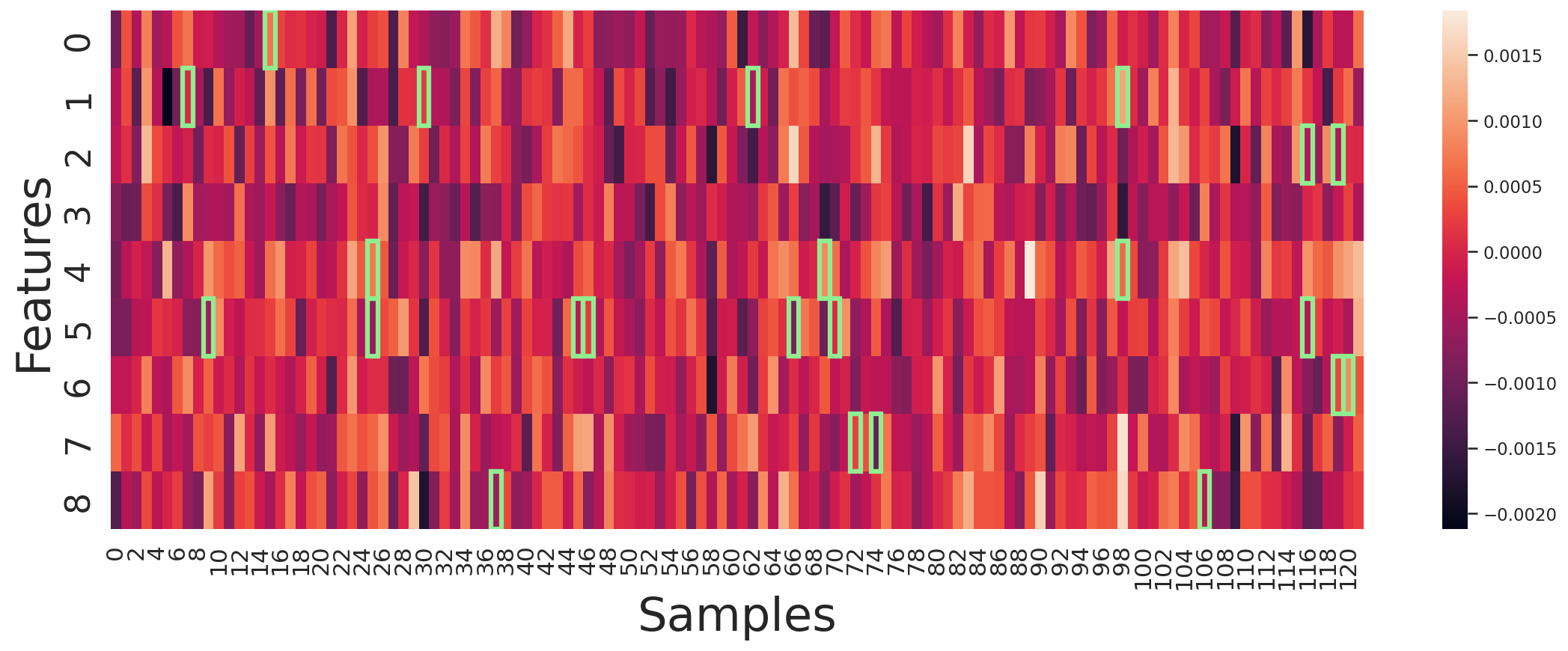}
  \caption{$\oned$ values for benign patients in the breast cancer dataset with randomly inserted outliers. The green border denotes a cell after an outlier value has been injected to that cell.}~\label{fig:bc_1d_out_heatmap}
  \end{center}
\end{figure}

In addition, we present a heatmap with injected outliers generated by $\oned$ to provide insights into the $\oned$ detection performance, which we show in Figure~\ref{fig:census_breast_outlier_detection}A). As we can observe the $\oned$ heatmap in Figure~\ref{fig:bc_1d_out_heatmap}, the values of injected outliers are scattered which explains why the detection rate by $\oned$ was suboptimal.

\subsubsection{Ablation Study on the Budget of Inserted Outliers}

In Figure~\ref{fig:census_breast_outlier_detection}A), we injected outlier values to $2\%$ of total cells. Here, we explore whether our $\alg$ method can still detect outliers on various different amount of outliers. Thus, we randomly inject $1\%,2\%,5\%,10\%,15\%$ of outlier values to the original breast cancer dataset and plot the detection rate. 

\begin{figure}[htb]
\begin{center}
  \includegraphics[width=300pt]{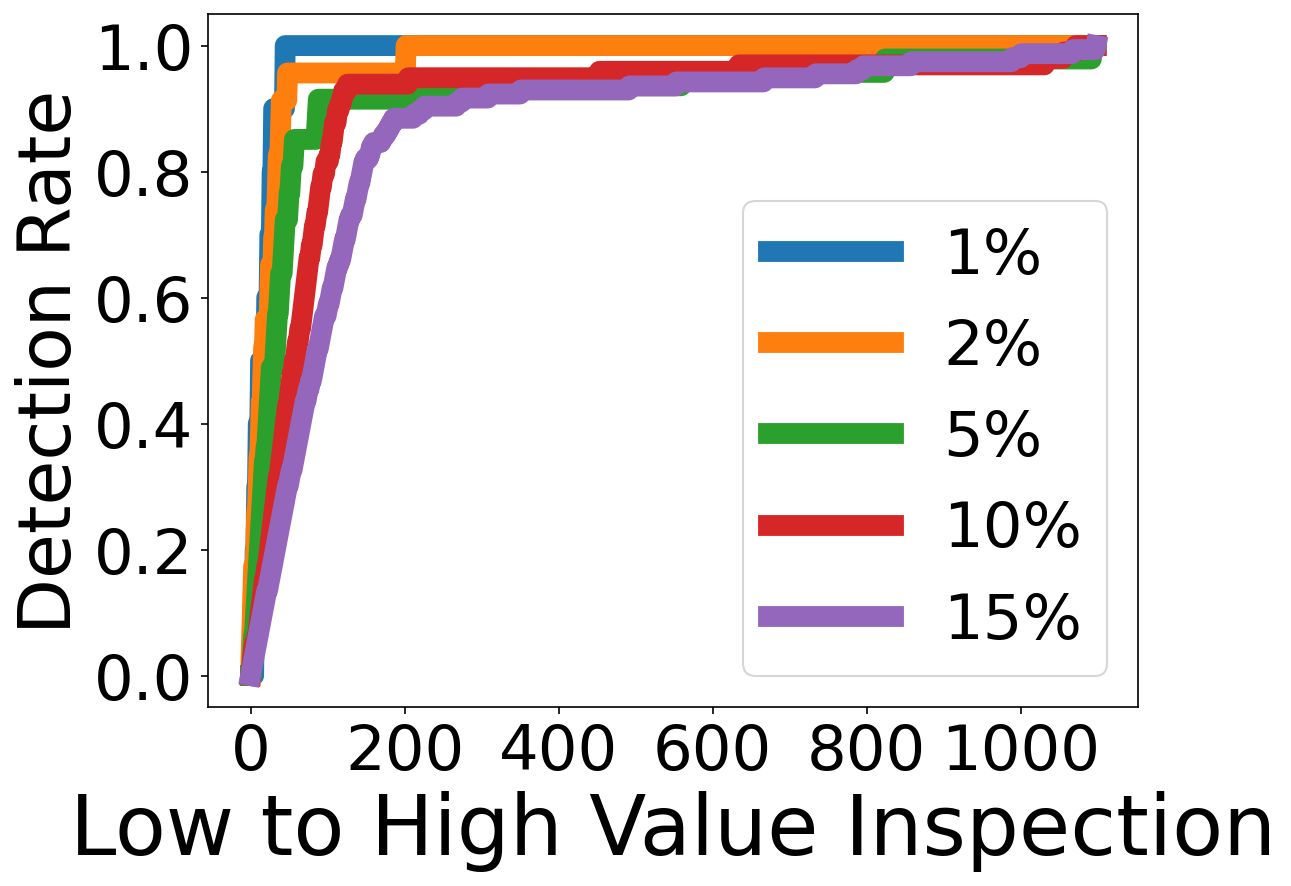}
  \caption{$\alg$ Detection rate of randomly inserted outliers in the breast cancer dataset over various injection rates.}~\label{fig:many_bc_outlier}
  \end{center}
\end{figure}

As we observe in Figure~\ref{fig:many_bc_outlier}, the detection rate of outliers is very high within the first $200$ inspected cells for every outlier injection rate. Further, we observe that with more outliers added to the dataset, our detection rate slightly decreases. It is indeed reasonable, since  as we inject more outliers in the dataset, the less uncommon these outliers are.




\subsection{Additional Details on \emph{Sub-matrix Valuation} experiment}





For the plots in Figure~\ref{fig:2d_vs_perf}, we have randomly split the Credit Default dataset into blocks. One of the random split is pictured in Figure~\ref{fig:block_split}. We randomly moved the horizontal and vertical lines and permuted separately rows and columns to create different possibilities for block splits.

\begin{figure}[htb]
\begin{center}
  \includegraphics[width=200pt]{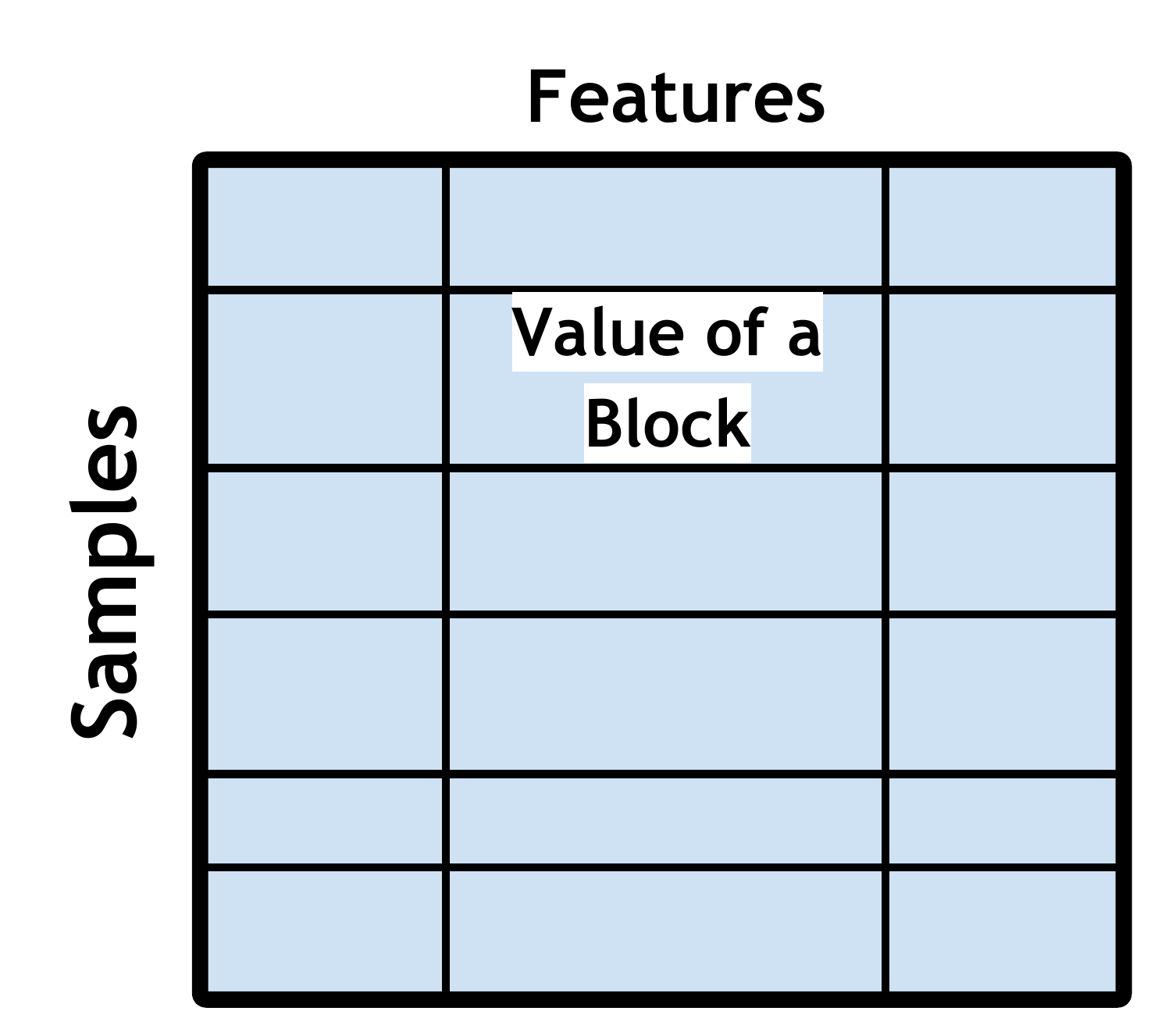}
  \caption{An example of a dataset split into blocks.}~\label{fig:block_split}
  \end{center}
\end{figure}

\subsection{Hardware}

In this work, we used an 8-Core Intel Xeon Processor E5-2620 v4 @ 2.20Ghz CPU server as a hardware platform.

\subsection{Code}

The code repository is available via this link \href{https://github.com/ruoxi-jia-group/2dshapley}{https://github.com/ruoxi-jia-group/2dshapley}.

\vfill
\end{appendices}